\PassOptionsToPackage{table}{xcolor}
\documentclass{article} %
\usepackage[T1]{fontenc}  %
\usepackage{iclr2027_conference,times}

\usepackage{amsmath,amsfonts,bm}

\def\eqref#1{equation~\ref{#1}}

\def\1{\bm{1}}

\DeclareMathAlphabet{\mathsfit}{\encodingdefault}{\sfdefault}{m}{sl}
\SetMathAlphabet{\mathsfit}{bold}{\encodingdefault}{\sfdefault}{bx}{n}

\usepackage{xcolor}
\usepackage{url}
\usepackage{booktabs}
\usepackage{multirow}
\usepackage{graphicx}
\usepackage{tabularx}
\usepackage{placeins}
\usepackage{needspace}
\usepackage{amssymb}   %
\usepackage{amsthm}
\newtheorem{lemma}{Lemma}
\newtheorem{proposition}{Proposition}
\newtheorem{corollary}{Corollary}
\usepackage{float}
\newfloat{algorithm}{tbp}{loa}
\floatname{algorithm}{Algorithm}
\usepackage{hyperref}

\hypersetup{hidelinks}

\title{DivOPD: Spread Wide, Look Close for \\ Asynchronous On-Policy Distillation \\ of Multi-turn Agents}

\author{%
Hanyang Wang$^{1,2*}$, Zeyuan Liu$^{2}$, Zhengyu Chen$^{2}$, Jingqing Ruan$^{2}$, Chaoxu Pang$^{2}$, Zhongda Su$^{2}$, \\
\textbf{Wulin Xie$^{3}$, Zhizhao Zeng$^{2}$, Ke Zeng$^{2}$, Tianxiang Zhao$^{4}$} \\[4pt]
{\normalfont $^{1}$University of Chicago \quad $^{2}$Meituan LongCat Interaction Team} \\
{\normalfont $^{3}$University of the Chinese Academy of Sciences} \\
{\normalfont $^{4}$The Hong Kong University of Science and Technology (Guangzhou)}
}


\makeatletter
\renewcommand{\subsection}{\@startsection{subsection}{2}{\z@}%
  {-0.95ex plus -0.2ex minus -0.1ex}%
  {0.35ex plus 0.1ex}%
  {\normalsize\sc\raggedright}}
\makeatother

\newcommand{\bst}[1]{\textbf{\boldmath #1}}
\newcommand{\snd}[1]{\underline{#1}}
\newcommand{\na}{\textendash}

\definecolor{tableheader}{RGB}{231,238,246}
\definecolor{tablesection}{RGB}{248,249,250}
\definecolor{tableaccent}{RGB}{249,234,232}
\definecolor{tablebest}{RGB}{249,242,228}
\definecolor{tabledivider}{RGB}{114,121,130}
\newcommand{\tabhead}{\rowcolor{tableheader}}

\newcommand{\tbest}[1]{\cellcolor{tablebest}\textbf{\boldmath #1}}

\definecolor{lighttabrule}{RGB}{178,193,206}
\newcommand{\resetlighttabrules}{\arrayrulecolor{black}}
\newcommand{\lighttablestyle}{%
  \arrayrulecolor{lighttabrule}%
  \aftergroup\resetlighttabrules
  \setlength{\heavyrulewidth}{0.7pt}%
  \setlength{\lightrulewidth}{0.45pt}%
  \setlength{\cmidrulewidth}{0.35pt}%
  \renewcommand{\arraystretch}{1.16}%
}

\iclrfinalcopy
\begin{document}
\raggedbottom

\maketitle
{\renewcommand{\thefootnote}{\fnsymbol{footnote}}\footnotetext[1]{Work done during an internship at Meituan LongCat Interaction Team.}}
\lhead{DivOPD: Spread Wide, Look Close for Asynchronous On-Policy Distillation of Multi-turn Agents}
\thispagestyle{fancy}

\begin{abstract}
On-policy distillation (OPD) trains student agents through teacher supervision
on their own interactions with an environment. However, in asynchronous
multi-turn training, arrival-order batching can allow a few early or long
rollouts to dominate learner updates while other valid rollouts become stale
before being used, wasting already-generated experience. To address this
problem, we introduce \textbf{DivOPD}, a simple learner-side batch-selection
method that spreads a fixed turn budget across more rollouts and, within each
rollout, prioritizes turns with larger cumulative teacher--student
disagreement. Turns without usable teacher feedback are excluded. The per-turn
loss and optimizer remain fixed; selection only changes which student-visited
turns receive training weight. For no-progress rollouts, an optional extension
briefly hands control to the teacher before returning it to the student.
Across six teacher--student settings on the simulated ALFWorld, ScienceWorld, and WebShop benchmarks,
with 1.5B--7B students, DivOPD raises cross-setting mean peak success rate from
$77.4$ to $84.4$ and mean success over the last five evaluations from $71.5$ to
$78.6$. It reaches all reported setting-specific targets with geometric-mean
speedups of $1.84\times$ in training tokens and $1.87\times$ in learner GPU time
relative to vanilla OPD. Teacher
intervention further raises this last-five mean to $82.4$ while retaining about
$1.7\times$ learner-GPU speedup over vanilla OPD.
Code will be released at \url{https://github.com/HanyangWang0418-oss/DivOPD}.
\end{abstract}

\section{Introduction}
\label{sec:intro}

OPD trains student agents through teacher supervision
on their own interactions with an environment. It is a natural fit for agents
because the student acts while a frozen teacher scores the states the student
actually visits~\citep{lin2020autoregressive,agarwal2023gkd}. In contrast,
distillation on fixed teacher-generated sequences~\citep{hinton2015distilling,kim2016sequence}
does not directly train on deviations that emerge during interaction. OPD
follows the online imitation-learning principle of supervising learner-visited
states~\citep{ross2011reduction} and reduces the exposure mismatch between
training and inference~\citep{bengio2015scheduled}.

However, asynchronous multi-turn OPD creates a learner-side allocation
problem: generating valid experience does not ensure that the learner will use
it. Explorers produce variable-length rollouts while the learner trains on
queued turns; model weights synchronize periodically, and turns expire after a
bounded policy age. If learner batches are filled in arrival order, a few early
or long rollouts can occupy many slots while other valid rollouts become stale
before contributing a gradient. The system then wastes interactions and
teacher scores that it has already paid to generate.

This asynchronous execution is increasingly common as foundation-model
post-training targets agents that reason, call tools, and interact with
environments over long, variable-duration
trajectories~\citep{xiao2026mimo,ma2026mopd,kimiteam2026k3,gao2025beyond}.
Strict rollout--update barriers are inefficient because an entire update can
wait for slow generation, tool execution, or teacher scoring. Recent
large-scale systems therefore span fully asynchronous rollout and
learning~\citep{fu2025areal,hu2026dora}, asynchronous teacher prefill
overlapped with rollout generation~\citep{ma2026mopd}, and partial rollouts
that continue across training
iterations~\citep{xiao2026mimo,kimiteam2026k3,zhou2025april}. These designs
differ in synchronization, but they make experience freshness a key
constraint. AsyncOPD exposes the same trade-off directly for distillation:
separating rollout generation from learning improves throughput but creates
experience from older policies~\citep{kang2026asyncopd}.

Unlike single-turn OPD, where each prompt produces one self-contained
completion, agentic OPD produces a multi-turn rollout: every
action changes the environment state and therefore the
observations and decisions available at later turns. Rollouts also vary in the
number and latency of model calls, tool executions, and environment steps.
As a result, their turns arrive as correlated, variable-length groups tied to
the same policy version, rather than as unrelated single-turn examples.

Vanilla OPD does not account for this grouped structure: it flattens each
multi-turn rollout into individual rows and fills learner batches in arrival
order, discarding the rollout as the natural allocation unit. Under a bounded
staleness window, batch construction therefore becomes an online allocation
decision that determines which rollouts receive any gradient before their
turns expire. It can also spend scarce slots on rows without a usable teacher
score.

Recent agentic OPD methods primarily improve experience before it reaches the
learner. They control rollout depth with staged schedules, early stopping, or probes;
rebalance supervision across turns; or alter visited states through scheduled
teacher intervention~\citep{wang2026tcod,ziheng2026less,zhou2026turnopd,li2026policy,chen2026look}.
These mechanisms address what experience is generated or how each turn is
supervised, often through rollout- or task-dependent schedules. They do not
directly address the separate learner-side question: how should a fixed batch
budget be allocated across already-generated, correlated rollouts before they
become stale? Better supervision cannot help a rollout that never reaches an
optimizer update.

Other asynchronous OPD methods correct, filter, or assign less weight to stale
samples~\citep{kang2026asyncopd,rang2026near,chen2026boldsymbol}. Open-MOPD
identifies a related token-budget imbalance across domains and refreshes stale
rewards across repeated inner updates~\citep{gao2026open}. These methods
control domain balance, sample freshness, or policy drift; they do not model how
correlated turns from multiple agent rollouts within one learner queue compete
for a fixed batch. Rollout-level allocation therefore has not been addressed even
when staleness is bounded.

Our training logs show that this allocation failure is large. Vanilla
OPD batches draw only $3.9$--$8.3$ effective rollouts, their largest rollout
supplies $24$--$40\%$ of batch tokens, and shared-buffer replays find that
$49$--$59\%$ of valid rollouts never enter an optimizer batch
(Figure~\ref{fig:clustering}, Appendix Table~\ref{tab:utilization-full}). We therefore
introduce \textbf{DivOPD}, which treats learner batch construction as online
allocation over correlated, expiring rollouts rather than flat sampling over
turns (Figure~\ref{fig:concept}). It removes turns without usable supervision,
limits each rollout's first-pass contribution, and selects high-disagreement
turns within each rollout. The last two steps serve different roles. Coverage
determines how many slots each rollout receives, while focus chooses which
turns fill those slots. The
selector reuses cached scores without an extra model pass and leaves vanilla
OPD's rollout generation, per-turn loss, and optimizer unchanged.

\begin{figure}[ht]
\centering
\includegraphics[width=\linewidth]{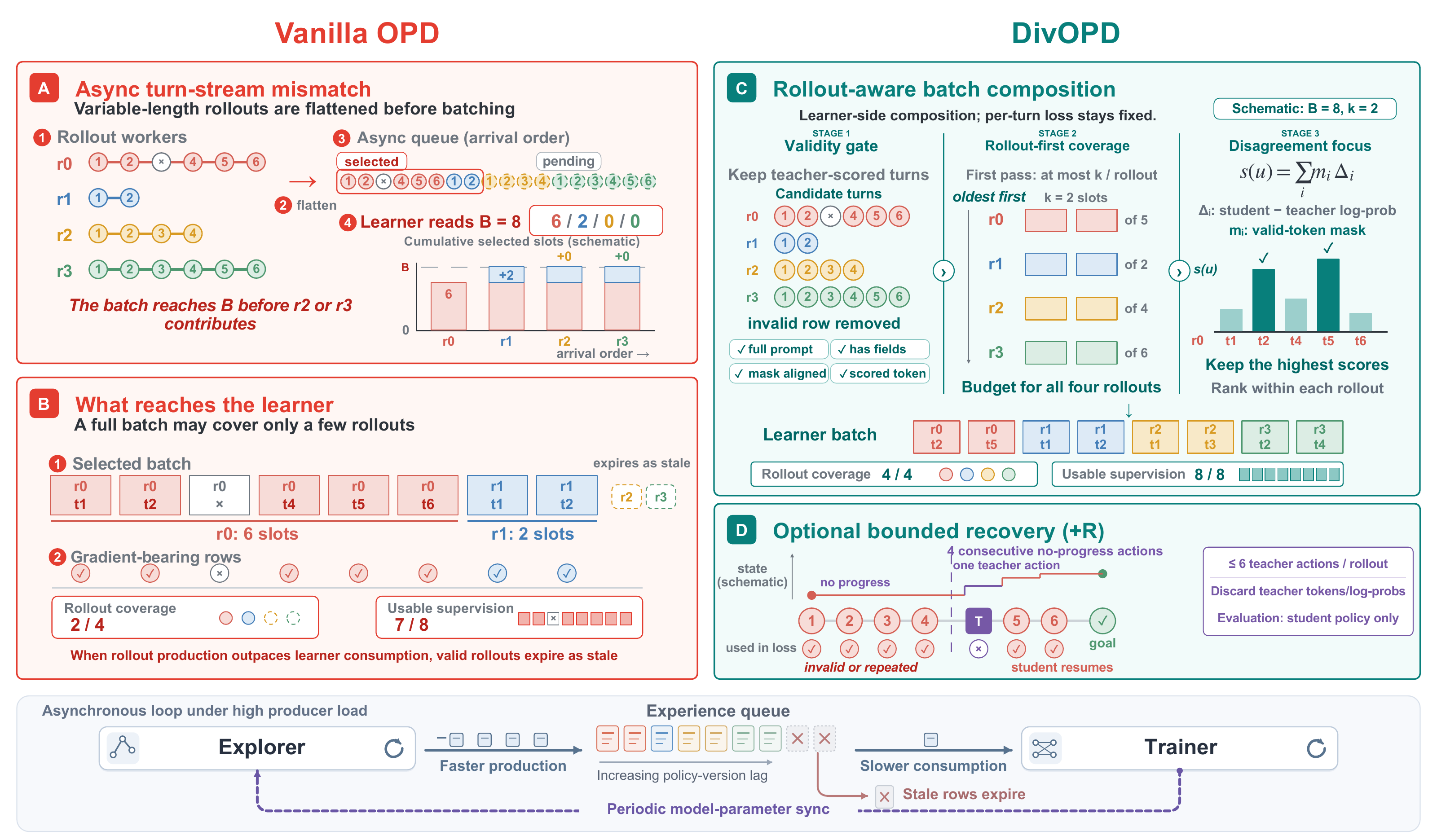}
\caption{\textbf{Batch composition in asynchronous multi-turn OPD.} (A) Rollout
workers enqueue variable-length interactions in arrival order. (B) When
production outpaces learner consumption, early or long rollouts can fill the
batch, include unscored rows, and leave others to expire as stale. (C) DivOPD
composes batches on the learner side through the three stages in
Section~\ref{sec:method}: validity gate, rollout-first coverage, and
within-rollout disagreement focus using $s(u)$
(Equation~\ref{eq:learning-need}). (D) When rollouts repeatedly make no
progress, bounded teacher recovery helps
the student continue; teacher actions are excluded
from the loss, and recovery is disabled at evaluation. Bottom: independent
exploration and training, with periodic model synchronization and stale-row expiry.}
\label{fig:concept}
\end{figure}

When a rollout repeatedly makes no progress, selecting among its turns does
not expose the student to new states. We therefore extend
DivOPD with bounded teacher recovery
($+$R): the teacher briefly intervenes, then the student's subsequent turns
return to the same batch composer.

\paragraph{Contributions.}
\textbf{(1)} We measure how arrival-order batching uses student experience in
asynchronous multi-turn OPD: it concentrates batches on a few
rollouts, spends slots on unscored rows, and lets many rollouts expire untrained
(Section~\ref{sec:diagnosis}, Appendix Table~\ref{tab:utilization-full}).
\textbf{(2)} We introduce DivOPD, a rollout-aware batch composer that retains
the per-turn distillation loss and optimizer while changing which scored turns
receive training weight (Section~\ref{sec:method}). Bounded teacher recovery
extends the same composer to help students resume after repeated no-progress turns.
\textbf{(3)} Across six teacher--student settings, we show success-rate gains
and geometric-mean trained-token and trainer-GPU speedups of $1.84\times$ and
$1.87\times$ over vanilla OPD. Controlled training runs and shared-buffer replays
separate the effects of coverage and focus (Section~\ref{sec:experiments}).

\section{Asynchronous Multi-turn OPD and Batch Allocation}
\label{sec:prelim}

\subsection{Learning from queued student turns}

\textbf{Asynchronous multi-turn OPD.} Rollout workers generate student
interactions and frozen-teacher scores independently of learner updates,
adding turns to a queue with rollout IDs and policy versions. The learner forms
fixed-size batches, discarding turns whose generating policy is too old.
A rollout is
$\rho=(o_1,a_1,\ldots,o_{T_\rho},a_{T_\rho})$, whose horizon $T_\rho$ depends on
the outcome. We write $\pi_\theta$ for the student and $\pi_\phi$ for the frozen
teacher, reserving $\tau$ for the target success rate. A batch contains
$N=B$ rows, one per student response turn; $B$ is the learner's per-update
budget.
For row $t$, let $y_{t,i}$ be the $i$-th response token and $h_{t,<i}$ the
prompt and tokens preceding it. The effective mask is
$m_{t,i}=m^{\mathrm{resp}}_{t,i}\wedge m^{\mathrm{valid}}_{t,i}$ with valid-token
set $V_t=\{i:m_{t,i}=1\}$: tokens that belong to the response and have valid
teacher scores. Rows with $|V_t|=0$ are \emph{dead}. With
distillation coefficient $\beta$, the stopped per-token advantage is
\begin{equation}
A_{t,i} = \beta\,\operatorname{sg}\!\left[
\log \pi_\phi(y_{t,i}\mid h_{t,<i})
- \log \pi_{\theta_{\mathrm{old}}}(y_{t,i}\mid h_{t,<i})
\right], \qquad i\in V_t.
\label{eq:opd-adv}
\end{equation}
Here $\theta_{\mathrm{old}}$ is the learner snapshot immediately before the
update; its log probabilities are recomputed on queued tokens, rather than
taken from their generating policy. $\operatorname{sg}$ stops gradients through
the score difference. The loss averages valid tokens within each row, then
averages all $N$ rows (seq-mean-token-mean). A dead row contributes zero but stays in the
denominator. Each batch receives one optimizer update, so the learner-snapshot
ratio is one at gradient evaluation and PPO clipping is inactive;
there is no entropy term, KL penalty, critic, or reference model
(Appendix~\ref{app:opd-objective}).

\subsection{Where the batch budget goes}
\label{sec:diagnosis}

Each queued turn is identified by rollout ID and turn index
$(\mathrm{rid},t)$. We ask two questions: how many rollouts are represented
in each batch, and how many are ever used before they expire? For a batch, let $p_r$
be the fraction of tokens from rollout $r$. The effective rollout count
$(\sum_r p_r^2)^{-1}$ measures token-weighted coverage, and $\max_r p_r$
measures the largest rollout's share.

\begin{figure}[!t]
\centering
\includegraphics[width=0.98\linewidth]{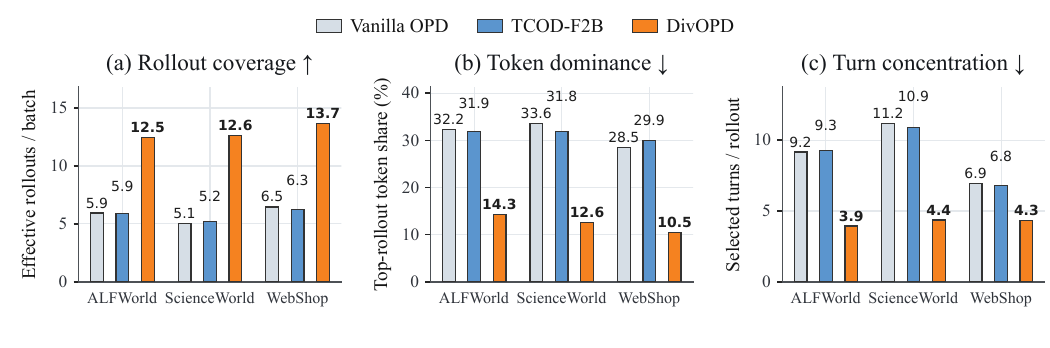}
\caption{\textbf{DivOPD changes batch composition across environments.}
Each bar averages the final-$40\%$ diagnostic over the two student sizes in
that environment. DivOPD covers more rollouts and reduces both concentration
measures in all three environments; TCOD-F2B remains close to vanilla OPD.
The six setting-level results appear in Appendix
Table~\ref{tab:coverage-cost-full}.}
\label{fig:clustering}
\end{figure}

\textbf{Concentration within an update.} At the setting level, vanilla OPD batches contain
only $3.9$--$8.3$ effective rollouts, and the largest rollout supplies
$24$--$40\%$ of batch tokens. In addition, $5.0$--$27.8\%$ of selected rows on ALFWorld and WebShop
have no teacher-scored token in the final training window. All ScienceWorld
rows pass the validity checks, allowing us to examine allocation without
dead-slot removal.

\textbf{Rollouts that expire unused.} Turns from one rollout arrive together.
With a finite staleness window, repeatedly selecting from early rollouts leaves
less time to use later arrivals. Shared-buffer replays show that roughly half of
valid rollouts are never selected by the arrival-order reader
(Appendix Table~\ref{tab:utilization-full}). This motivates spreading batch slots across
rollouts before ranking their turns. Section~\ref{sec:experiments} tests
whether changing this allocation improves learning.

\section{DivOPD}
\label{sec:method}

Motivated by the batch composition in Figure~\ref{fig:clustering}, DivOPD
composes learner batches from this asynchronous queue of already scored turns.
It first combines turns retained from earlier updates (pending turns)
with new arrivals in a pool of at most $cB$ turns. From this pool, the three
stages shown from left to right in Figure~\ref{fig:concept}C remove invalid
turns, allocate slots across rollouts, and rank turns within each rollout. The
batch contains $B$ valid turns whenever the pool has enough.

\subsection{Stage 1: validity gate}
\label{sec:validity}

A turn is valid only if its prompt was not truncated and it contains the
tokens, teacher log probabilities, and action mask needed by the loss. If a
teacher-valid mask is supplied, it must align with the action mask and share
at least one valid position. Turns that fail these checks are discarded
before batch selection.

In four ALFWorld replay buffers, every rejected row is a response turn with a
truncated prompt from a failed rollout and zero scored tokens. Removing these rows frees
slots for valid turns. Longer context windows or truncating the history before
generation could also prevent such rows (Appendix~\ref{app:replay-audit}).

\subsection{Stage 2: rollout-first coverage}
\label{sec:coverage}

To prevent one long rollout from filling the batch, the composer groups
valid turns by rollout ID and visits the longest-waiting groups first.
On the first pass through these groups, each rollout contributes at most $k$ turns,
which guarantees $\lceil B/k\rceil$ distinct rollouts whenever that many groups
are available. If the batch is not full, each
additional pass raises the cap by one. Thus $k$ limits the first pass, not
the final number of turns per rollout. Unlike shuffling, this changes which
turns enter the batch, not just their order.

Unselected valid turns are kept for later updates until they expire. This
allows more rollouts to enter each batch and gives waiting rollouts another
chance to be used. The cap counts turns, not tokens, so it does not bound a
rollout's token share when response lengths vary.

\subsection{Stage 3: within-rollout disagreement focus}
\label{sec:focus}

Once a rollout receives slots, focus chooses the turns that fill them. Let
$v(u)$ be the student policy version that generated turn $u$. DivOPD ranks its
valid turns by the cached score over its valid tokens $V_u$,
\begin{equation}
s(u) = \sum_{i\in V_u}\left[
\log\pi_{\theta_{v(u)}}(y_i\mid h_{u,<i})
-\log\pi_\phi(y_i\mid h_{u,<i})\right],
\label{eq:learning-need}
\end{equation}
in descending order, breaking ties by $(\mathrm{rid},t)$, on every
pass. The score measures cumulative, not per-token, disagreement: it equals
the valid-token count times the mean token score. With a complete response
mask, its expectation is the sequence-level reverse KL at the generating
policy~\citep{gu2024minillm,agarwal2023gkd}; a partial mask restricts it to
scored tokens. A small signed sum may occur because positive and negative token
scores cancel.

Focus uses this score only to rank turns within allocated rollout slots;
it neither takes slots from other rollouts nor adds a length-based loss
weight. The score reuses cached log probabilities without another model pass.
It is a selection heuristic, not the current learner's row loss or gradient
norm: the learner can be up to two versions ahead.
Appendix~\ref{app:utilization} audits ranking stability under this lag;
Section~\ref{sec:ablation-main} tests its added benefit at fixed allocation.

\subsection{Optional bounded teacher recovery}
\label{sec:recovery}

Batch selection cannot help a rollout that repeatedly makes no progress.
DivOPD$+$R therefore lets the teacher act for at most $m$ turns after $p$
no-progress turns, then returns control to the student. Teacher actions receive
no distillation weight, subsequent student turns use the same composer, and
recovery is disabled during evaluation. Unlike core DivOPD, $+$R changes the
visited trajectory. Appendix~\ref{app:ablation} defines the trigger and gives
the per-setting patience, takeover cap, and warmup values.

\subsection{Training distribution and invariants}
\label{sec:drainage}
\label{sec:invariants}

The per-turn loss is fixed, but selection still changes what the update
averages over. Coverage limits the influence of rollout length on allocation,
and focus shifts weight toward high-disagreement
turns inside each rollout. We do not add importance weights to undo either
change; Appendix~\ref{app:removal} states this shift precisely.

DivOPD-base uses the same gate and rollout
allocations with uniform within-rollout sampling; DivOPD replaces that
sampling with focus. Both use Equation~\ref{eq:opd-adv} with unit
per-turn weights, and the same reduction, optimizer, and rollout procedure
as vanilla OPD. Appendix Table~\ref{tab:sampler-config} collects the batch,
pool, staleness, pending, cap, and loss settings. Algorithm~\ref{alg:divopd-sampler} gives the full procedure;
core DivOPD changes only queue consumption.

\section{Experiments}
\label{sec:experiments}

\subsection{Setup and metrics}
\label{sec:setup}

We use the same asynchronous explorer--learner pipeline for all methods and
evaluate two student scales per environment:
ALFWorld~\citep{shridhar2020alfworld} (Qwen3-1.7B/4B),
WebShop~\citep{yao2022webshop} (Qwen2.5-3B/7B-Instruct), and
ScienceWorld~\citep{wang2022scienceworld} (Qwen2.5-1.5B/3B-Instruct).
Teachers are environment-specific GiGPO-Qwen2.5-7B
models~\citep{feng2026group} on ALFWorld and WebShop, and the ScienceWorld
checkpoint released with Embodied-Planner-R1~\citep{fei2025unleashing} on
ScienceWorld. Baselines are vanilla
OPD~\citep{lin2020autoregressive,agarwal2023gkd},
TCOD-F2B~\citep{wang2026tcod}, and
TurnOPD~\citep{zhou2026turnopd}, all with the same learner batch size. DivOPD-focus is an earlier
focus-first composer reported for reference; it does not share the sampler of
the other DivOPD variants, so it is not a controlled ablation of focus.
DivOPD$+$R follows Section~\ref{sec:recovery}; exact settings appear in
Appendix~\ref{app:ablation}, and evaluation remains student-only.

\begin{figure}[!b]
\centering
\includegraphics[width=\linewidth]{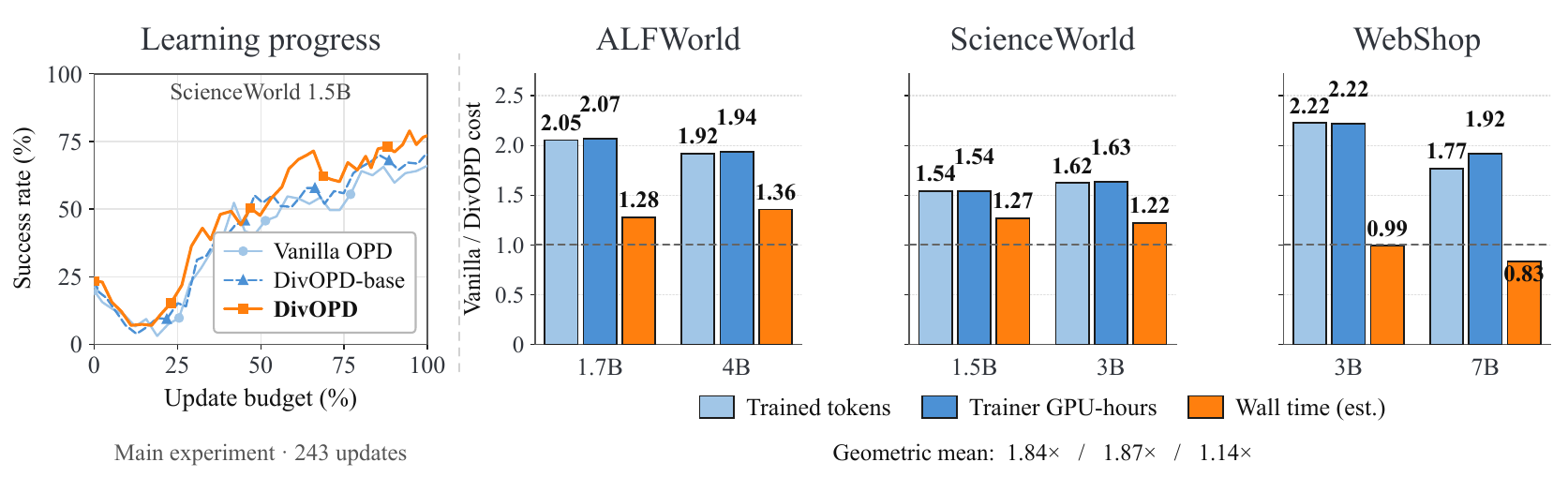}
\caption{\textbf{Efficiency depends on which cost is measured.}
Left: unsmoothed ScienceWorld 1.5B learning curves over a common 243-update
budget. Right: vanilla-to-DivOPD cost ratios to the first evaluation at
$\tau$; values above one favor DivOPD. Wall time includes experience reads
(Appendix Table~\ref{tab:resource-efficiency}).}
\label{fig:resource-main}
\end{figure}

Vanilla reads $B$ candidates per update; DivOPD can inspect $cB$ pending and
fresh turns. A no-cap control with uniform within-rollout sampling matches DivOPD-base's candidate
access (Section~\ref{sec:ablation-main}). Appendix~\ref{app:sampler-algorithm}
gives all numerical training and evaluation settings.

We report peak success rate (SR) and final-five SR, the mean over the last
five evaluation checkpoints; $\pm$ denotes the standard deviation across
those checkpoints. For efficiency, the target success rate $\tau$ is a
multiple of 5 at or below vanilla OPD's peak SR. For either cumulative
learner-side resource $C$ (training tokens or trainer GPU-hours), we report
$C_{\mathrm{vanilla}}(\tau)/C_{\mathrm{method}}(\tau)$, where each cost is
accumulated through the first evaluation reaching $\tau$. Values above one
favor the method; a dash means that the method did not reach $\tau$. We also
report normalized area under the success-rate
curve over a common learner-update budget (nAUC), equivalently the mean SR over
that budget, as a target-free comparison (Section~\ref{sec:efficiency}). The
trainer-loop wall-time estimate, which includes experience reads, improves by
$1.14\times$ on geometric average, although WebShop 3B and 7B show no
wall-time gain ($0.99\times$ and $0.83\times$); Appendix~\ref{app:cost} gives
the full breakdown.

\subsection{Success rates across six settings}

Table~\ref{tab:main-results} compares accuracy and learner-side efficiency to target.

\newcommand{\fullmainresultstable}{%
\begin{table}[ht]
\caption{\textbf{Full results across six teacher--student settings.} Peak SR is
the best checkpoint; Final-5 is the mean $\pm$ standard deviation over the
last five evaluation checkpoints. Best values are highlighted in bold;
second-best are underlined.}
\label{tab:main-results-full}
\centering
\fontsize{8.0}{10.0}\selectfont
\setlength{\tabcolsep}{2.4pt}
\renewcommand{\arraystretch}{1.03}
\begin{tabular*}{\linewidth}{@{\extracolsep{\fill}}llrrrrr@{}}
\toprule
\textbf{Setting} & \textbf{Method}
& \multicolumn{2}{c}{\textbf{Accuracy}}
& \multicolumn{3}{c}{\textbf{Efficiency to $\tau$}} \\
\cmidrule(lr){3-4}\cmidrule(lr){5-7}
{} & {}
& \textbf{Peak}$\uparrow$ & \textbf{Final-5}$\uparrow$
& \textbf{Updates}$\downarrow$ & \textbf{Trainer GPU h}$\downarrow$ & \textbf{GPU speedup}$\uparrow$ \\
\midrule
\multirow{7}{*}{\shortstack[l]{\textbf{ALFWorld}\\Qwen3-4B\\$\tau{=}80$}}
& Vanilla OPD
& 86.43 & 84.57 $\pm$ 1.48 & 184 & 13.51 & 1.00$\times$ \\
& TCOD-F2B
& 87.14 & 83.00 $\pm$ 1.28 & 182 & 8.95 & 1.51$\times$ \\
& TurnOPD
& 87.14 & 86.43 $\pm$ 0.51 & 161 & 8.50 & 1.59$\times$ \\
\cmidrule(lr){2-7}
& DivOPD-focus
& 89.29 & 85.29 $\pm$ 3.10 & 148 & 7.97 & 1.70$\times$ \\
& DivOPD-base
& \tbest{92.14} & \snd{88.86 $\pm$ 2.29} & 154 & 10.61 & 1.27$\times$ \\
& \textbf{DivOPD}
& \snd{91.43} & 88.00 $\pm$ 2.74 & \snd{127} & \tbest{6.98} & \tbest{1.94}$\times$ \\
& \textbf{DivOPD$+$R}
& \tbest{92.14} & \tbest{89.00 $\pm$ 2.51} & \tbest{126} & \snd{7.06} & \snd{1.91}$\times$ \\
\specialrule{0.55pt}{2.2pt}{1.0pt}
\multirow{7}{*}{\shortstack[l]{\textbf{ALFWorld}\\Qwen3-1.7B\\$\tau{=}70$}}
& Vanilla OPD
& 77.86 & 74.00 $\pm$ 3.18 & 208 & 8.21 & 1.00$\times$ \\
& TCOD-F2B
& 72.86 & 68.43 $\pm$ 3.09 & 211 & 5.04 & 1.63$\times$ \\
& TurnOPD
& 78.57 & \snd{75.71 $\pm$ 2.72} & 195 & 4.66 & 1.76$\times$ \\
\cmidrule(lr){2-7}
& DivOPD-focus
& \snd{80.00} & 71.00 $\pm$ 6.62 & 214 & 5.62 & 1.46$\times$ \\
& DivOPD-base
& 74.29 & 72.86 $\pm$ 1.43 & 189 & 6.08 & 1.35$\times$ \\
& \textbf{DivOPD}
& 79.29 & 71.86 $\pm$ 7.29 & \snd{155} & \snd{3.97} & \snd{2.07}$\times$ \\
& \textbf{DivOPD$+$R}
& \tbest{82.86} & \tbest{78.29 $\pm$ 1.20} & \tbest{150} & \tbest{3.75} & \tbest{2.19}$\times$ \\
\specialrule{0.9pt}{2.8pt}{1.0pt}
\multirow{7}{*}{\shortstack[l]{\textbf{WebShop}\\Qwen2.5-3B\\$\tau{=}60$}}
& Vanilla OPD
& 60.94 & 43.44 $\pm$ 13.02 & 147 & 10.02 & 1.00$\times$ \\
& TCOD-F2B
& 67.19 & 61.72 $\pm$ 4.42 & 104 & 6.11 & 1.64$\times$ \\
& TurnOPD
& 68.75 & 58.59 $\pm$ 8.80 & 114 & 5.44 & 1.84$\times$ \\
\cmidrule(lr){2-7}
& DivOPD-focus
& 70.31 & 66.41 $\pm$ 3.45 & \snd{99} & \snd{4.17} & \snd{2.40}$\times$ \\
& DivOPD-base
& \snd{78.12} & \snd{73.12 $\pm$ 4.64} & 113 & 5.33 & 1.88$\times$ \\
& \textbf{DivOPD}
& \snd{78.12} & 67.81 $\pm$ 3.96 & 106 & 4.52 & 2.22$\times$ \\
& \textbf{DivOPD$+$R}
& \tbest{79.69} & \tbest{76.41 $\pm$ 4.04} & \tbest{87} & \tbest{3.84} & \tbest{2.61}$\times$ \\
\specialrule{0.55pt}{2.2pt}{1.0pt}
\multirow{7}{*}{\shortstack[l]{\textbf{WebShop}\\Qwen2.5-7B\\$\tau{=}85$}}
& Vanilla OPD
& 87.50 & 82.50 $\pm$ 3.11 & 88 & 8.69 & 1.00$\times$ \\
& TCOD-F2B
& 89.06 & 83.59 $\pm$ 3.79 & 97 & 6.23 & 1.40$\times$ \\
& TurnOPD
& \snd{89.84} & 85.94 $\pm$ 2.47 & 106 & 6.19 & 1.41$\times$ \\
\cmidrule(lr){2-7}
& DivOPD-focus
& \snd{89.84} & \snd{86.88 $\pm$ 2.17} & \tbest{82} & \tbest{4.36} & \tbest{1.99}$\times$ \\
& DivOPD-base
& 88.28 & 78.44 $\pm$ 4.37 & 116 & 7.14 & 1.22$\times$ \\
& \textbf{DivOPD}
& \tbest{90.62} & 82.66 $\pm$ 5.25 & \snd{86} & \snd{4.53} & \snd{1.92}$\times$ \\
& \textbf{DivOPD$+$R}
& \snd{89.84} & \tbest{88.28 $\pm$ 2.71} & 101 & 5.20 & 1.67$\times$ \\
\specialrule{0.9pt}{2.8pt}{1.0pt}
\multirow{7}{*}{\shortstack[l]{\textbf{ScienceWorld}\\Qwen2.5-3B\\$\tau{=}85$}}
& Vanilla OPD
& 85.55 & 80.78 $\pm$ 4.09 & 233 & 2.71 & 1.00$\times$ \\
& TCOD-F2B
& 87.89 & 83.98 $\pm$ 2.58 & 228 & \snd{2.26} & \snd{1.20}$\times$ \\
& TurnOPD
& 83.20 & 80.86 $\pm$ 2.72 & \na & \na & \na \\
\cmidrule(lr){2-7}
& DivOPD-focus
& 77.34 & 74.69 $\pm$ 2.20 & \na & \na & \na \\
& DivOPD-base
& 86.72 & 82.66 $\pm$ 1.50 & \snd{190} & 2.73 & 0.99$\times$ \\
& \textbf{DivOPD}
& \snd{88.28} & \tbest{85.47 $\pm$ 2.49} & \tbest{131} & \tbest{1.66} & \tbest{1.63}$\times$ \\
& \textbf{DivOPD$+$R}
& \tbest{90.23} & \snd{85.08 $\pm$ 2.92} & 214 & 2.52 & 1.07$\times$ \\
\specialrule{0.55pt}{2.2pt}{1.0pt}
\multirow{7}{*}{\shortstack[l]{\textbf{ScienceWorld}\\Qwen2.5-1.5B\\$\tau{=}65$}}
& Vanilla OPD
& 66.02 & 63.75 $\pm$ 2.49 & 211 & 1.97 & 1.00$\times$ \\
& TCOD-F2B
& 71.88 & 66.33 $\pm$ 3.77 & 221 & 1.55 & 1.27$\times$ \\
& TurnOPD
& 71.48 & 68.75 $\pm$ 3.65 & \snd{174} & \tbest{1.28} & \tbest{1.54}$\times$ \\
\cmidrule(lr){2-7}
& DivOPD-focus
& 66.80 & 61.88 $\pm$ 4.71 & 242 & 2.07 & 0.95$\times$ \\
& DivOPD-base
& 70.31 & 67.50 $\pm$ 2.20 & 195 & 1.98 & 0.99$\times$ \\
& \textbf{DivOPD}
& \tbest{78.91} & \snd{76.09 $\pm$ 2.23} & \tbest{149} & \tbest{1.28} & \tbest{1.54}$\times$ \\
& \textbf{DivOPD$+$R}
& \snd{78.12} & \tbest{77.42 $\pm$ 0.75} & 179 & \snd{1.52} & \snd{1.30}$\times$ \\
\bottomrule
\end{tabular*}
\vspace{1pt}

\parbox{\linewidth}{\fontsize{7.0}{7.8}\selectfont\textit{Protocol.}
Every run assigns two GPUs to the trainer. GPU hours accumulate trainer-step
time through the first evaluation reaching $\tau$; speedups are relative to
vanilla OPD. Cap settings appear in Section~\ref{sec:setup}.}
\end{table}
}

\begin{table}[!t]
\definecolor{mainbest}{RGB}{245,194,193}
\definecolor{mainsecond}{RGB}{217,217,252}
\renewcommand{\tbest}[1]{\begingroup\setlength{\fboxsep}{1.0pt}\colorbox{mainbest}{\strut\textbf{\boldmath #1}}\endgroup}
\renewcommand{\snd}[1]{\begingroup\setlength{\fboxsep}{1.0pt}\colorbox{mainsecond}{\strut #1}\endgroup}
\caption{\textbf{Main results.} Peak/Final-5 are best-checkpoint/final-five mean
SR (\%); $\pm$ is the standard deviation over the last five checkpoints. Headers give
teacher/student-initial SR. Tok./GPU are
$C_{\mathrm{vanilla}}(\tau)/C_{\mathrm{method}}(\tau)$ for cumulative learner
training tokens and trainer GPU-hours, respectively. Larger is better; a dash
means that the method did not reach $\tau$. Red/purple
mark best/second-best;
DivOPD-focus is an uncontrolled reference.}
\label{tab:main-results}
\centering
\fontsize{6.6}{7.4}\selectfont
\setlength{\tabcolsep}{1.2pt}
\renewcommand{\arraystretch}{0.89}
\begin{tabular*}{\linewidth}{@{\extracolsep{\fill}}lrrrrrrrr@{}}
\toprule
\textbf{ALFWorld}
& \multicolumn{4}{c}{\textbf{Qwen3-1.7B} ($T$/Init: 87.86/7.86; $\tau{=}70$)}
& \multicolumn{4}{c}{\textbf{Qwen3-4B} ($T$/Init: 87.86/26.43; $\tau{=}80$)} \\
\cmidrule(lr){2-5}\cmidrule(lr){6-9}
\textbf{Method} & \textbf{Peak}$\uparrow$ & \textbf{Final-5}$\uparrow$ & \textbf{Tok.}$\uparrow$ & \textbf{GPU}$\uparrow$
& \textbf{Peak}$\uparrow$ & \textbf{Final-5}$\uparrow$ & \textbf{Tok.}$\uparrow$ & \textbf{GPU}$\uparrow$ \\
\midrule
Vanilla OPD  & 77.86 & 74.00 $\pm$ 3.18 & 1.00$\times$ & 1.00$\times$ & 86.43 & 84.57 $\pm$ 1.48 & 1.00$\times$ & 1.00$\times$ \\
TCOD-F2B     & 72.86 & 68.43 $\pm$ 3.09 & 1.62$\times$ & 1.63$\times$ & 87.14 & 83.00 $\pm$ 1.28 & 1.49$\times$ & 1.51$\times$ \\
TurnOPD      & 78.57 & \snd{75.71 $\pm$ 2.72} & 1.77$\times$ & 1.76$\times$ & 87.14 & 86.43 $\pm$ 0.51 & 1.57$\times$ & 1.59$\times$ \\
\cmidrule(lr){1-9}
DivOPD-focus & \snd{80.00} & 71.00 $\pm$ 6.62 & 1.50$\times$ & 1.46$\times$ & 89.29 & 85.29 $\pm$ 3.10 & 1.69$\times$ & 1.70$\times$ \\
DivOPD-base  & 74.29 & 72.86 $\pm$ 1.43 & 1.35$\times$ & 1.35$\times$ & \tbest{92.14} & \snd{88.86 $\pm$ 2.29} & 1.27$\times$ & 1.27$\times$ \\
\textbf{DivOPD} & 79.29 & 71.86 $\pm$ 7.29 & \snd{2.05}$\times$ & \snd{2.07}$\times$ & \snd{91.43} & 88.00 $\pm$ 2.74 & \tbest{1.92}$\times$ & \tbest{1.94}$\times$ \\
\textbf{DivOPD$+$R} & \tbest{82.86} & \tbest{78.29 $\pm$ 1.20} & \tbest{2.17}$\times$ & \tbest{2.19}$\times$ & \tbest{92.14} & \tbest{89.00 $\pm$ 2.51} & \snd{1.90}$\times$ & \snd{1.91}$\times$ \\
\specialrule{0.8pt}{2.8pt}{1.2pt}

\textbf{WebShop}
& \multicolumn{4}{c}{\textbf{Qwen2.5-3B} ($T$/Init: 85.94/1.56; $\tau{=}60$)}
& \multicolumn{4}{c}{\textbf{Qwen2.5-7B} ($T$/Init: 85.94/14.06; $\tau{=}85$)} \\
\cmidrule(lr){2-5}\cmidrule(lr){6-9}
\textbf{Method} & \textbf{Peak}$\uparrow$ & \textbf{Final-5}$\uparrow$ & \textbf{Tok.}$\uparrow$ & \textbf{GPU}$\uparrow$
& \textbf{Peak}$\uparrow$ & \textbf{Final-5}$\uparrow$ & \textbf{Tok.}$\uparrow$ & \textbf{GPU}$\uparrow$ \\
\midrule
Vanilla OPD  & 60.94 & 43.44 $\pm$ 13.02 & 1.00$\times$ & 1.00$\times$ & 87.50 & 82.50 $\pm$ 3.11 & 1.00$\times$ & 1.00$\times$ \\
TCOD-F2B     & 67.19 & 61.72 $\pm$ 4.42 & 1.64$\times$ & 1.64$\times$ & 89.06 & 83.59 $\pm$ 3.79 & 1.27$\times$ & 1.40$\times$ \\
TurnOPD      & 68.75 & 58.59 $\pm$ 8.80 & 1.86$\times$ & 1.84$\times$ & \snd{89.84} & 85.94 $\pm$ 2.47 & 1.30$\times$ & 1.41$\times$ \\
\cmidrule(lr){1-9}
DivOPD-focus & 70.31 & 66.41 $\pm$ 3.45 & \snd{2.37}$\times$ & \snd{2.40}$\times$ & \snd{89.84} & \snd{86.88 $\pm$ 2.17} & \tbest{1.91}$\times$ & \tbest{1.99}$\times$ \\
DivOPD-base  & \snd{78.12} & \snd{73.12 $\pm$ 4.64} & 1.88$\times$ & 1.88$\times$ & 88.28 & 78.44 $\pm$ 4.37 & 1.11$\times$ & 1.22$\times$ \\
\textbf{DivOPD} & \snd{78.12} & 67.81 $\pm$ 3.96 & 2.22$\times$ & 2.22$\times$ & \tbest{90.62} & 82.66 $\pm$ 5.25 & \snd{1.77}$\times$ & \snd{1.92}$\times$ \\
\textbf{DivOPD$+$R} & \tbest{79.69} & \tbest{76.41 $\pm$ 4.04} & \tbest{2.60}$\times$ & \tbest{2.61}$\times$ & \snd{89.84} & \tbest{88.28 $\pm$ 2.71} & 1.52$\times$ & 1.67$\times$ \\
\specialrule{0.8pt}{2.8pt}{1.2pt}

\textbf{ScienceWorld}
& \multicolumn{4}{c}{\textbf{Qwen2.5-1.5B} ($T$/Init: 93.75/21.48; $\tau{=}65$)}
& \multicolumn{4}{c}{\textbf{Qwen2.5-3B} ($T$/Init: 93.75/45.31; $\tau{=}85$)} \\
\cmidrule(lr){2-5}\cmidrule(lr){6-9}
\textbf{Method} & \textbf{Peak}$\uparrow$ & \textbf{Final-5}$\uparrow$ & \textbf{Tok.}$\uparrow$ & \textbf{GPU}$\uparrow$
& \textbf{Peak}$\uparrow$ & \textbf{Final-5}$\uparrow$ & \textbf{Tok.}$\uparrow$ & \textbf{GPU}$\uparrow$ \\
\midrule
Vanilla OPD  & 66.02 & 63.75 $\pm$ 2.49 & 1.00$\times$ & 1.00$\times$ & 85.55 & 80.78 $\pm$ 4.09 & 1.00$\times$ & 1.00$\times$ \\
TCOD-F2B     & 71.88 & 66.33 $\pm$ 3.77 & 1.21$\times$ & 1.27$\times$ & 87.89 & 83.98 $\pm$ 2.58 & \snd{1.21}$\times$ & \snd{1.20}$\times$ \\
TurnOPD      & 71.48 & 68.75 $\pm$ 3.65 & \snd{1.49}$\times$ & \tbest{1.54}$\times$ & 83.20 & 80.86 $\pm$ 2.72 & \na & \na \\
\cmidrule(lr){1-9}
DivOPD-focus & 66.80 & 61.88 $\pm$ 4.71 & 0.94$\times$ & 0.95$\times$ & 77.34 & 74.69 $\pm$ 2.20 & \na & \na \\
DivOPD-base  & 70.31 & 67.50 $\pm$ 2.20 & 1.00$\times$ & 0.99$\times$ & 86.72 & 82.66 $\pm$ 1.50 & 1.00$\times$ & 0.99$\times$ \\
\textbf{DivOPD} & \tbest{78.91} & \snd{76.09 $\pm$ 2.23} & \tbest{1.54}$\times$ & \tbest{1.54}$\times$ & \snd{88.28} & \tbest{85.47 $\pm$ 2.49} & \tbest{1.62}$\times$ & \tbest{1.63}$\times$ \\
\textbf{DivOPD$+$R} & \snd{78.12} & \tbest{77.42 $\pm$ 0.75} & 1.29$\times$ & \snd{1.30}$\times$ & \tbest{90.23} & \snd{85.08 $\pm$ 2.92} & 1.03$\times$ & 1.07$\times$ \\
\bottomrule
\end{tabular*}
\vspace{1pt}

\parbox{\linewidth}{\fontsize{6.8}{7.5}\selectfont\textit{Protocol.}
All runs use two trainer GPUs. Full updates and GPU hours appear in Appendix
Table~\ref{tab:main-results-full}.}
\end{table}

Across the six settings, DivOPD raises mean peak SR from $77.4$ to $84.4$ and
final-five SR from $71.5$ to $78.6$. Peak SR improves in all six settings
and final-five SR in five. On ALFWorld 1.7B, where late
checkpoints fluctuate, DivOPD$+$R reaches the highest final-five SR
($78.29\pm1.20$).

\subsection{Learner efficiency and elapsed time}
\label{sec:efficiency}

Our primary efficiency claim concerns learner-side resources. We report
learner training tokens, trainer GPU-hours, and elapsed trainer-loop time separately
(Figure~\ref{fig:resource-main}, right)
because reducing learner work need not accelerate the whole pipeline equally.

DivOPD reaches every target with fewer learner training tokens and trainer GPU-hours
than vanilla OPD: the geometric-mean speedups are $1.84\times$ and
$1.87\times$, respectively. On ALFWorld and WebShop, the gain combines fewer
updates with fewer tokens per update; on ScienceWorld, where all rows are
valid, it comes primarily from fewer updates. Appendix~\ref{app:cost} gives the
cost breakdown.

The trainer-loop wall-time estimate improves by $1.14\times$ on geometric
average, but the ratio is $0.99\times$ on WebShop 3B and $0.83\times$ on
WebShop 7B. Once learner updates become cheaper, rollout generation and data
reads account for a larger share of elapsed time (Appendix~\ref{app:cost}).

\begin{figure}[H]
\centering
\includegraphics[width=\linewidth]{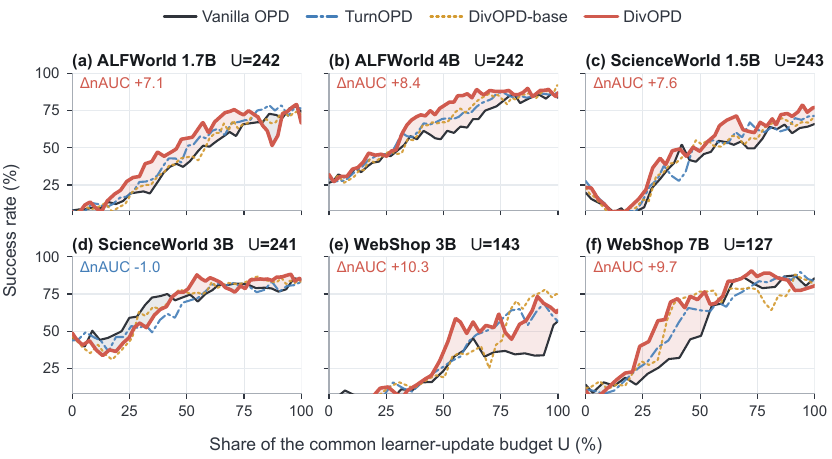}
\caption{\textbf{Learning curves at a common update budget.} Each panel aligns
the archived runs on the largest learner-update budget $U$ they share. Curves
are unsmoothed evaluation success rates. The normalized area under each curve
(nAUC) is its mean SR over this budget; shading shows the area between DivOPD
and vanilla OPD. Appendix~\ref{app:fixed-budget} gives the protocol.}
\label{fig:fixed-budget}
\end{figure}

\textbf{A target-free comparison.} We also align the archived runs on a
common learner-update budget $U$ (Figure~\ref{fig:fixed-budget}). DivOPD has
the higher nAUC in five of six settings, providing evidence of improvement
without selecting a target $\tau$. Figure~\ref{fig:resource-main} (left)
enlarges the ScienceWorld 1.5B comparison.

\textbf{Sensitivity to asynchronous producer load.} Table~\ref{tab:async-pressure}
varies the explorer rollout batch on ALFWorld 1.7B over 150 learner updates,
with sampler settings fixed. The two readers share
the producer configuration. To show the pressure that actually accumulates,
we report the measured queue arrival rate divided by the learner drain rate.
A value above one means experience enters the queue faster than the learner
removes it; this ratio is an outcome of the reader, not a matched control.
Vanilla OPD and DivOPD-base are nearly tied in final SR at light load,
where both measured ratios are below one. At medium
and high load, DivOPD-base raises final
SR by $2.9$ and $10.0$ points and greatly reduces the raw stale-row count.
Relative to vanilla, its nAUC is higher by $0.4$, $0.7$, and $0.9$ points from light to high load,
so the advantage is visible over the fixed training budget rather than only
at the final checkpoint. This pattern is consistent with the intended use
case: allocating limited learner slots when production exceeds consumption.

\begin{table}[!htbp]
\caption{\textbf{Stress test under increasing asynchronous producer load.}
ALFWorld Qwen3-1.7B over 150 learner updates.
Arrival/drain is the measured queue arrival rate divided by the learner drain
rate; values above one indicate a growing backlog. SR and nAUC are percentages;
stale rows are raw counts. Bold marks the better value within each load level.}
\label{tab:async-pressure}
\centering
\fontsize{8.6}{9.6}\selectfont
\setlength{\tabcolsep}{4.0pt}
\renewcommand{\arraystretch}{1.04}
\lighttablestyle
\begin{tabular*}{\linewidth}{@{\extracolsep{\fill}}llrrrrr@{}}
\toprule
\textbf{Load} & \textbf{Reader} & \textbf{Rollout bs} & \textbf{Arrival/drain}
& \textbf{Final SR}$\uparrow$ & \textbf{nAUC@150}$\uparrow$ & \textbf{Stale rows}$\downarrow$ \\
\midrule
\multirow{2}{*}{Low}
& Vanilla OPD  & 2  & 0.77 & \textbf{52.1} & 25.0 & 240 \\
& DivOPD-base  & 2  & 0.55 & 52.0 & \textbf{25.4} & \textbf{53} \\
\cmidrule(lr){1-7}
\multirow{2}{*}{Medium}
& Vanilla OPD  & 6  & 1.35 & 53.6 & 25.3 & 8,035 \\
& DivOPD-base  & 6  & 0.88 & \textbf{56.5} & \textbf{26.0} & \textbf{4,237} \\
\cmidrule(lr){1-7}
\multirow{2}{*}{High}
& Vanilla OPD  & 16 & 2.27 & 49.3 & 26.9 & 30,175 \\
& DivOPD-base  & 16 & 0.89 & \textbf{59.3} & \textbf{27.8} & \textbf{9,766} \\
\bottomrule
\end{tabular*}
\end{table}

\subsection{Separating coverage from focus}
\label{sec:ablation-main}

DivOPD differs from vanilla OPD in three ways: it filters invalid turns, reads
more candidates, and changes turn selection. We therefore use separate
controls to test the cap and focus (Figure~\ref{fig:component-controls}).

\textbf{Cap with the same candidate access.} We compare DivOPD-base with an
uncapped ($k=\infty$) control. Both use the same $cB$ candidate-pool size, validity
gate, oldest-first order, pending mechanism, and uniform within-rollout
sampling. Across the six settings, finite $k$ has higher peak SR throughout
(mean $+2.1$ points) and higher final-five SR in five (mean $+3.4$), with all
other selection rules fixed (Table~\ref{tab:matched-access-cap}).

\textbf{Selection with the same candidate access.} We give vanilla OPD the same $256$-turn candidate
stream on one setting per environment and compare it with rollout-first
selection (Table~\ref{tab:access-control}). Rollout-first selection yields
peak SR $1.4$--$4.3$ points higher and Best-5 SR $1.7$--$5.6$ points higher in all
three runs. At the common cutoff, final SR is higher on ScienceWorld, tied on
ALFWorld, and lower on WebShop. Because these are separate runs truncated at a
common learner-version cutoff, their absolute SR is not comparable to
Table~\ref{tab:main-results}.

\textbf{Focus with the same rollout allocation.} DivOPD and DivOPD-base assign
slots in the same way but choose different turns within each rollout.
Focus raises mean peak SR by $2.8$ points and reduces trainer-GPU
time-to-target in all six settings, giving a further $1.38$--$1.60\times$
dataset-mean speedup over DivOPD-base. Final-five SR rises in only three
settings (mean $+1.4$ points).

\textbf{Cap sensitivity.} Peak SR varies by up to $7.9$ points within a
setting over $k\in\{3,\ldots,6\}$ (Appendix Table~\ref{tab:k-sweep}).
The best cap varies by setting, so choosing $k$ matters.

\subsection{Which rollouts and turns are selected?}
\label{sec:coverage-cost}

\textbf{Coverage brings more rollouts into training.} We replay the selection
rules on shared buffers so that every method sees the same experience.
Appendix Table~\ref{tab:utilization-full} tracks whether each valid rollout
contributes any turn over 100 updates. Rollout-first selection lowers the unused fraction from
$49$--$59\%$ to $7$--$14\%$ and reduces the concentration of selected turns.
DivOPD-base and DivOPD allocate rollouts identically here; focus changes which
turns are taken, not how many rollouts are reached.

\textbf{Focus selects different gradients.} Holding the candidate pool and the
rollout allocation fixed, we compare selected gradients with the mean gradient
of all valid candidate turns. Focus increases the measured projection onto this
pool gradient in all 24 audited updates (Appendix~\ref{app:geometry}). This is
a diagnostic before the optimizer update, not a guarantee of better learning.

In the ALFWorld 1.7B replay, focus selects earlier turns from failed rollouts
without changing the success/failure mix
(Appendix Table~\ref{tab:replay-depth-outcome}). Pairwise gradient cosines stay
near zero for every composer: this diagnostic does not show a clear change in
within-batch gradient similarity.

\begin{figure}[H]
\centering
\includegraphics[width=0.96\linewidth]{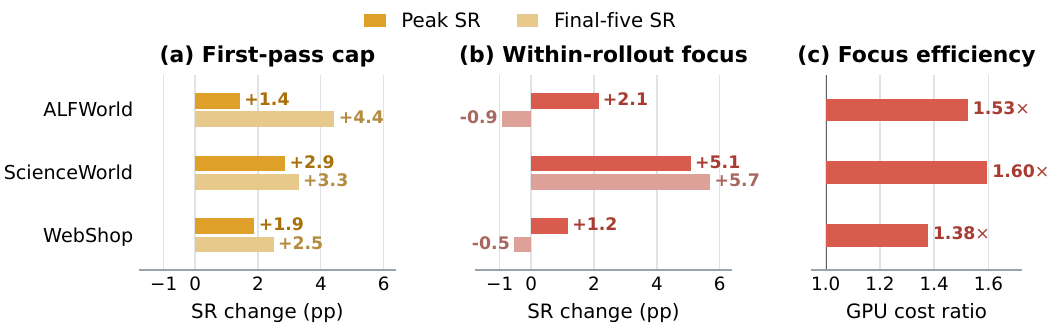}
\caption{\textbf{The cap improves both metrics; focus improves peak SR and learner efficiency.}
Each bar averages the two model sizes within one dataset. (a) DivOPD-base
minus the matched-access $k=\infty$ control isolates the cap. (b) DivOPD minus
DivOPD-base isolates focus at fixed rollout allocation; positive values favor
focus. (c) $C_{\mathrm{base}}(\tau)/C_{\mathrm{DivOPD}}(\tau)$ for trainer
GPU-hours; values above one mean that focus reaches $\tau$ with less trainer
compute. Setting-level results are in
Tables~\ref{tab:matched-access-cap} and~\ref{tab:main-results}.}
\label{fig:component-controls}
\end{figure}

\textbf{Disagreement is more evenly spread across selected rollouts.}
Figure~\ref{fig:kl-stability} tracks the standard deviation of the total
disagreement from each rollout in a learner batch. It is lower under DivOPD
than vanilla OPD through most of training on ALFWorld and ScienceWorld.

\begin{figure}[!htbp]
\centering
\includegraphics[width=0.96\linewidth]{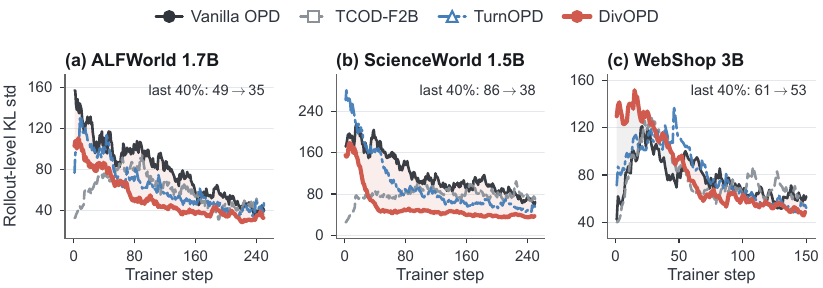}
\caption{\textbf{Cross-rollout disagreement, one setting per environment.}
Curves show the EMA ($\alpha=0.12$) of the within-batch standard deviation of
rollout-summed disagreement; labels give the final-$40\%$ mean (vanilla
$\to$ DivOPD). Red/gray shading marks where DivOPD is lower/higher. All settings appear in
Figure~\ref{fig:kl-stability-all}.}
\label{fig:kl-stability}
\end{figure}

\textbf{Extending selection to no-progress rollouts.} Adding recovery to DivOPD
raises final-five mean SR in five of six settings (cross-setting mean
$78.6\to82.4$) and peak SR in four, while retaining a geometric-mean
trainer-GPU speedup of about $1.7\times$ over vanilla OPD. These learner-side
ratios exclude teacher-generation cost (Appendix Figure~\ref{fig:ablation-recovery}).

\section{Related Work}
\label{sec:layers}

\textbf{Agentic OPD.} Generative distillation changes the sequence source or
objective~\citep{gu2024minillm,ko2024distillm}; TurnOPD balances turns~\citep{zhou2026turnopd};
TCOD and early stopping control rollout length~\citep{wang2026tcod,ziheng2026less};
and Guided-OPD and FutureBridge-OPD change visited states through teacher
intervention~\citep{li2026policy,chen2026look}. DivOPD instead selects which scored turns
enter each asynchronous learner batch (Appendix Table~\ref{tab:layers}).

\textbf{Multi-teacher and asynchronous OPD.} Multi-teacher OPD and Open-MOPD
route supervision or balance token budgets across domains~\citep{ma2026mopd,gao2026open}.
AsyncOPD, Near-Policy, and $f$-OPD control policy lag through correction,
filtering, or freshness~\citep{kang2026asyncopd,rang2026near,chen2026boldsymbol}.
These operate at the domain or sample level; DivOPD allocates a turn budget
across correlated rollouts in one asynchronous learner queue.

\textbf{Data selection and asynchronous learning.} Prioritized selection favors informative
examples~\citep{loshchilov2015online,schaul2015prioritized,katharopoulos2018not,mindermann2022prioritized},
DAPO drops zero-gradient groups~\citep{yu2026dapo}, and asynchronous actor--learner
systems control policy lag~\citep{espeholt2018impala,fu2025areal,hu2026dora}.
DivOPD instead separates rollout allocation from turn selection when queued,
variable-length rollouts compete for a fixed learner budget.

\section{Conclusion and Future Work}

Arrival-order batching in asynchronous multi-turn OPD can let early or long
rollouts dominate a fixed learner budget, leaving valid experience unused.
DivOPD treats this as a batch-composition problem: it filters unusable turns,
caps each rollout's first-pass contribution, and focuses allocated slots on
high-disagreement turns. The per-turn loss and optimizer stay unchanged;
selection changes the data-weighted training objective.

Across six Qwen-family settings on three benchmarks, DivOPD raises
mean peak SR from $77.4$ to $84.4$ and final-five SR from $71.5$ to $78.6$,
while achieving geometric-mean speedups of $1.84\times$ in learner training
tokens and $1.87\times$ in trainer GPU-hours. Controls support the cap and
broader coverage; focus improves peak SR and time-to-target, while the stress
test shows that gains grow with producer load. Recovery further raises mean
final-five SR to $82.4$ while retaining about $1.7\times$ trainer-GPU speedup.
These gains require neither a learned selector nor a new per-turn loss.

Batch construction is therefore part of the learning algorithm in asynchronous
multi-turn OPD, not neutral systems plumbing. Evidence remains limited to
simulated benchmarks and Qwen-family models; total pipeline gains also depend
on rollout and teacher costs. Future work should broaden these settings,
measure end-to-end resources, and adapt allocation to queue pressure.
\label{sec:limitations}

\subsection*{AI use statement}

Generative AI tools were used to draft portions of the manuscript and to
assist with language polishing and organization. The authors supplied the
scientific content and evidence, reviewed and revised all AI-assisted text,
and verified quantitative statements against experiment logs and source
artifacts. The authors take full responsibility for the final manuscript.

\subsection*{Ethics statement}

This work uses public simulated benchmarks and involves no human participants
or personal data. It studies training efficiency rather than deployment in
open-world systems. Teacher intervention can make an agent more capable, so
deployment in settings with real users or high-impact actions would require
task-specific safety evaluation and human review beyond the experiments
reported here.

\subsection*{Reproducibility statement}

Section~\ref{sec:method} specifies the composer and all of its default
hyperparameters; Section~\ref{sec:experiments} specifies the
environments, student--teacher pairs, and the budget-to-target protocol.
DivOPD-base, DivOPD, and DivOPD$+$R share one sampler implementation;
DivOPD-focus is documented separately. Every method row of Appendix
Table~\ref{tab:main-results-full} corresponds to a single configuration file.
The code and these configuration files will be released at \url{https://github.com/HanyangWang0418-oss/DivOPD}.

\begingroup
\footnotesize
\setlength{\bibsep}{0.5pt}
\bibliographystyle{iclr2027_conference}
\bibliography{iclr2027_conference}

@article{hinton2015distilling,
  title={Distilling the knowledge in a neural network},
  author={Hinton, Geoffrey and Vinyals, Oriol and Dean, Jeff},
  journal={arXiv preprint arXiv:1503.02531},
  year={2015}
}

@inproceedings{kim2016sequence,
  title={Sequence-level knowledge distillation},
  author={Kim, Yoon and Rush, Alexander M},
  booktitle={Proceedings of the 2016 conference on empirical methods in natural language processing},
  pages={1317--1327},
  year={2016}
}

@article{bengio2015scheduled,
  title={Scheduled sampling for sequence prediction with recurrent neural networks},
  author={Bengio, Samy and Vinyals, Oriol and Jaitly, Navdeep and Shazeer, Noam},
  journal={Advances in neural information processing systems},
  volume={28},
  year={2015}
}

@inproceedings{lin2020autoregressive,
  title={Autoregressive knowledge distillation through imitation learning},
  author={Lin, Alexander and Wohlwend, Jeremy and Chen, Howard and Lei, Tao},
  booktitle={Proceedings of the 2020 Conference on Empirical Methods in Natural Language Processing (EMNLP)},
  pages={6121--6133},
  year={2020}
}

@inproceedings{gu2024minillm,
  title={Minillm: Knowledge distillation of large language models},
  author={Gu, Yuxian and Dong, Li and Wei, Furu and Huang, Minlie},
  booktitle={International Conference on Learning Representations},
  volume={2024},
  pages={32694--32717},
  year={2024}
}

@article{ko2024distillm,
  title={Distillm: Towards streamlined distillation for large language models},
  author={Ko, Jongwoo and Kim, Sungnyun and Chen, Tianyi and Yun, Se-Young},
  journal={arXiv preprint arXiv:2402.03898},
  year={2024}
}

@inproceedings{wang2026text2grad,
  title={Text2grad: Reinforcement learning from natural language feedback},
  author={Wang, Hanyang and Wang, Lu and Zhang, Chaoyun and Mao, Tianjun and Qin, Si and Lin, Qingwei and Rajmohan, Saravan and Zhang, Dongmei},
  booktitle={International Conference on Learning Representations},
  volume={2026},
  pages={50627--50664},
  year={2026}
}

@article{feng2026group,
  title={Group-in-group policy optimization for llm agent training},
  author={Feng, Lang and Xue, Zhenghai and Liu, Tingcong and An, Bo},
  journal={Advances in Neural Information Processing Systems},
  volume={38},
  pages={46375--46408},
  year={2026}
}

@article{wang2026bipace,
  title={BiPACE: Bisimulation-Guided Policy Optimization with Action Counterfactual Estimation for LLM Agents},
  author={Wang, Hanyang and Ren, Weijieying and Zhang, Yuxiang and Cao, Ding and Zeng, Zhizhao and Zeng, Ke and Zhao, Tianxiang},
  journal={arXiv preprint arXiv:2606.25556},
  year={2026}
}

@article{xia2026skillrl,
  title={Skillrl: Evolving agents via recursive skill-augmented reinforcement learning},
  author={Xia, Peng and Chen, Jianwen and Wang, Hanyang and Liu, Jiaqi and Zeng, Kaide and Wang, Yu and Han, Siwei and Zhou, Yiyang and Zhao, Xujiang and Chen, Haifeng and others},
  journal={arXiv preprint arXiv:2602.08234},
  year={2026}
}

@article{agarwal2023gkd,
  title={Gkd: Generalized knowledge distillation for auto-regressive sequence models},
  author={Agarwal, Rishabh and Vieillard, Nino and Stanczyk, Piotr and Ramos, Sabela and Geist, Matthieu and Bachem, Olivier},
  journal={arXiv preprint arXiv:2306.13649},
  volume={12},
  year={2023}
}

@inproceedings{ross2011reduction,
  title={A reduction of imitation learning and structured prediction to no-regret online learning},
  author={Ross, St{\'e}phane and Gordon, Geoffrey and Bagnell, Drew},
  booktitle={Proceedings of the fourteenth international conference on artificial intelligence and statistics},
  pages={627--635},
  year={2011},
  organization={JMLR Workshop and Conference Proceedings}
}

@article{wang2026tcod,
  title={Tcod: Exploring temporal curriculum in on-policy distillation for multi-turn autonomous agents},
  author={Wang, Jiaqi and Zhang, Wenhao and Shi, Weijie and Li, Yaliang and Cheng, James},
  journal={arXiv preprint arXiv:2604.24005},
  year={2026}
}

@article{ziheng2026less,
  title={Less is more: Early stopping rollout for on-policy distillation},
  author={Ziheng, Zhou and Li, Jiaqi and Tang, Huacong and Wu, Ying Nian and Terzopoulos, Demetri},
  journal={arXiv preprint arXiv:2605.27028},
  year={2026}
}

@article{li2026policy,
  title={On-Policy Distillation with Curriculum Turn-level Guidance for Multi-turn Agents},
  author={Li, Gengsheng and Zheng, Mao and Song, Mingyang and Liu, Ruiqi and Yang, Tianyu and Sun, Jie and Zhong, Qiyong and Guo, Haiyun and Fang, Junfeng and Zhang, Dan and others},
  journal={arXiv preprint arXiv:2606.15912},
  year={2026}
}

@article{zhou2026turnopd,
  title={TurnOPD: Making On-Policy Distillation Turn-Aware for Efficient Long-Horizon Agent Training},
  author={Zhou, Yuhang and Zheng, Kai and Li, Haoling and Peng, Dengyun and Xu, Can and Chen, Jingjing},
  journal={arXiv preprint arXiv:2607.05804},
  year={2026}
}

@article{chen2026look,
  title={Look Ahead Before You Distill: Future Trajectory Validation of Teacher Guidance for Agentic On-Policy Distillation},
  author={Chen, Chishui and Fan, Yaoyou and Sun, Te and Yang, Yi and Sun, Chenghao and Mao, Delin and Qiao, Hongbo and Zhang, Zuowei and Wang, Junxi and Sun, Chenxing and others},
  journal={arXiv preprint arXiv:2608.01953},
  year={2026}
}

@article{fu2025areal,
  title={{AReaL}: A Large-Scale Asynchronous Reinforcement Learning System for Language Reasoning},
  author={Fu, Wei and Gao, Jiaxuan and Shen, Xujie and Zhu, Chen and Mei, Zhiyu and He, Chuyi and Xu, Shusheng and Wei, Guo and Mei, Jun and Wang, Jiashu and Yang, Tongkai and Yuan, Binhang and Wu, Yi},
  journal={arXiv preprint arXiv:2505.24298},
  year={2025}
}

@article{ma2026mopd,
  title={{MOPD}: Multi-Teacher On-Policy Distillation for Capability Integration in {LLM} Post-Training},
  author={Ma, Wenhan and Wei, Jianyu and Zhao, Liang and Zhang, Hailin and Xiao, Bangjun and Li, Lei and Yang, Qibin and Gao, Bofei and Wang, Yudong and Li, Rang and Dong, Jinhao and Sui, Zhifang and Luo, Fuli},
  journal={arXiv preprint arXiv:2606.30406},
  year={2026}
}

@article{kimiteam2026k3,
  title={{Kimi K3}: Open Frontier Intelligence},
  author={{Kimi Team}},
  journal={arXiv preprint arXiv:2607.24653},
  year={2026}
}

@article{kang2026asyncopd,
  title={{AsyncOPD}: How Stale Can On-Policy Distillation Be?},
  author={Kang, Wonjun and Galim, Kevin and Oh, Seunghyuk and Kang, Minjun and Park, Sanghyun and Kim, Donghoon and Lee, Minjae and Kim, Minseo and Tiwari, Rishabh and Zeng, Yuchen and Koo, Hyung Il and Lee, Kangwook},
  journal={arXiv preprint arXiv:2606.24143},
  year={2026}
}

@article{shridhar2020alfworld,
  title={Alfworld: Aligning text and embodied environments for interactive learning},
  author={Shridhar, Mohit and Yuan, Xingdi and C{\^o}t{\'e}, Marc-Alexandre and Bisk, Yonatan and Trischler, Adam and Hausknecht, Matthew},
  journal={arXiv preprint arXiv:2010.03768},
  year={2020}
}

@article{wang2022scienceworld,
  title={Scienceworld: Is your agent smarter than a 5th grader?, 2022},
  author={Wang, Ruoyao and Jansen, Peter and C{\^o}t{\'e}, Marc-Alexandre and Ammanabrolu, Prithviraj},
  journal={URL https://arxiv. org/abs/2203.07540},
  year={2022}
}

@article{yao2022webshop,
  title={Webshop: Towards scalable real-world web interaction with grounded language agents},
  author={Yao, Shunyu and Chen, Howard and Yang, John and Narasimhan, Karthik},
  journal={Advances in Neural Information Processing Systems},
  volume={35},
  pages={20744--20757},
  year={2022}
}

@article{schaul2015prioritized,
  title={Prioritized experience replay},
  author={Schaul, Tom and Quan, John and Antonoglou, Ioannis and Silver, David},
  journal={arXiv preprint arXiv:1511.05952},
  year={2015}
}

@article{loshchilov2015online,
  title={Online batch selection for faster training of neural networks},
  author={Loshchilov, Ilya and Hutter, Frank},
  journal={arXiv preprint arXiv:1511.06343},
  year={2015}
}

@inproceedings{katharopoulos2018not,
  title={Not all samples are created equal: Deep learning with importance sampling},
  author={Katharopoulos, Angelos and Fleuret, Fran{\c{c}}ois},
  booktitle={International conference on machine learning},
  pages={2525--2534},
  year={2018},
  organization={PMLR}
}

@inproceedings{mindermann2022prioritized,
  title={Prioritized training on points that are learnable, worth learning, and not yet learnt},
  author={Mindermann, S{\"o}ren and Brauner, Jan M and Razzak, Muhammed T and Sharma, Mrinank and Kirsch, Andreas and Xu, Winnie and H{\"o}ltgen, Benedikt and Gomez, Aidan N and Morisot, Adrien and Farquhar, Sebastian and others},
  booktitle={International Conference on Machine Learning},
  pages={15630--15649},
  year={2022},
  organization={PMLR}
}

@article{yu2026dapo,
  title={Dapo: An open-source llm reinforcement learning system at scale},
  author={Yu, Qiying and Zhang, Zheng and Zhu, Ruofei and Yuan, Yufeng and Zuo, Xiaochen and Yue, Yu and Dai, Weinan and Fan, Tiantian and Liu, Gaohong and Liu, Lingjun and others},
  journal={Advances in Neural Information Processing Systems},
  volume={38},
  pages={113222--113244},
  year={2026}
}

@inproceedings{espeholt2018impala,
  title={Impala: Scalable distributed deep-rl with importance weighted actor-learner architectures},
  author={Espeholt, Lasse and Soyer, Hubert and Munos, Remi and Simonyan, Karen and Mnih, Vlad and Ward, Tom and Doron, Yotam and Firoiu, Vlad and Harley, Tim and Dunning, Iain and others},
  booktitle={International conference on machine learning},
  pages={1407--1416},
  year={2018},
  organization={PMLR}
}

@article{fei2025unleashing,
  title={Unleashing embodied task planning ability in llms via reinforcement learning},
  author={Fei, Zhaoye and Ji, Li and Wang, Siyin and Shi, Junhao and Gong, Jingjing and Qiu, Xipeng},
  journal={arXiv preprint arXiv:2506.23127},
  year={2025}
}

@article{xiao2026mimo,
  title={{MiMo-V2-Flash} technical report},
  author={Xiao, Bangjun and Xia, Bingquan and Yang, Bo and Gao, Bofei and Shen, Bowen and Zhang, Chen and He, Chenhong and Lou, Chiheng and Luo, Fuli and Wang, Gang and others},
  journal={arXiv preprint arXiv:2601.02780},
  year={2026}
}

@article{rang2026near,
  title={Near-Policy: Accelerating On-Policy Distillation via Asynchronous Generation and Selective Packing},
  author={Rang, Miao and Bi, Zhenni and Zhou, Hang and Han, Kai and Wang, Xuechun and Xiao, An and Chen, Xinghao and Wang, Yunhe and Chen, Hanting},
  journal={arXiv preprint arXiv:2605.05940},
  year={2026}
}

@article{chen2026boldsymbol,
  title={{f}-OPD: Stabilizing Long-Horizon On-Policy Distillation with Freshness-Aware Control},
  author={Chen, Xianwei and Zhang, Shimin and Wu, Jibin},
  journal={arXiv preprint arXiv:2605.17862},
  year={2026}
}

@article{gao2025beyond,
  title={Beyond ten turns: Unlocking long-horizon agentic search with large-scale asynchronous {RL}},
  author={Gao, Jiaxuan and Fu, Wei and Xie, Minyang and Xu, Shusheng and He, Chuyi and Mei, Zhiyu and Zhu, Banghua and Wu, Yi},
  journal={arXiv preprint arXiv:2508.07976},
  year={2025}
}

@article{zhou2025april,
  title={{APRIL}: Active partial rollouts in reinforcement learning to tame long-tail generation},
  author={Zhou, Yuzhen and Li, Jiajun and Su, Yusheng and Ramesh, Gowtham and Zhu, Zilin and Long, Xiang and Zhao, Chenyang and Pan, Jin and Yu, Xiaodong and Wang, Ze and others},
  journal={arXiv preprint arXiv:2509.18521},
  year={2025}
}

@article{hu2026dora,
  title={{DORA}: A scalable asynchronous reinforcement learning system for language model training},
  author={Hu, Tianhao and Liu, Xiangcheng and Miao, Yuchun and Xiao, Youshao and Zang, Hongyu and Zheng, Yang and Huang, Xuan and Ding, Jinrui and Zhang, Yufei and Yang, Yu and others},
  journal={arXiv preprint arXiv:2604.26256},
  year={2026}
}

@article{gao2026open,
  title={Open-{MOPD}: Diagnosing and Fixing Capability Imbalance in Multi-Teacher On-Policy Distillation},
  author={Gao, Huan-ang and Chi, Haohan and Yan, Yong and Feng, Shiyuan and Wu, Hanlin and Jiang, Zheng and He, Bingxiang and Ma, Wei-Ying and Zhang, Ya-Qin and Zhou, Hao},
  journal={arXiv preprint arXiv:2608.19098},
  year={2026}
}
\endgroup

\clearpage
\appendix

\setlength{\parskip}{1.2pt plus 0.3pt minus 0.2pt}
\setlength{\textfloatsep}{6pt plus 1pt minus 1pt}
\setlength{\floatsep}{5pt plus 1pt minus 1pt}
\setlength{\intextsep}{5pt plus 1pt minus 1pt}
\setlength{\abovecaptionskip}{3pt}

\noindent\begin{minipage}{\linewidth}
\centering
\small
\textbf{Appendix Contents}\\[2pt]
\renewcommand{\arraystretch}{1.08}
\begin{tabularx}{\linewidth}{@{}X@{\hspace{1.5em}}X@{}}
\hyperref[app:sampler-algorithm]{\textbf{A}\quad Sampler-side batch composition}\dotfill\pageref{app:sampler-algorithm}
& \hyperref[app:fixed-budget]{\textbf{B}\quad Fixed-budget results}\dotfill\pageref{app:fixed-budget} \\
\hyperref[app:main-results-full]{\textbf{C}\quad Full main results}\dotfill\pageref{app:main-results-full}
& \hyperref[app:batch-diagnostics]{\textbf{D}\quad Batch-composition diagnostics}\dotfill\pageref{app:batch-diagnostics} \\
\hyperref[app:opd-objective]{\textbf{E}\quad OPD objective and analysis}\dotfill\pageref{app:opd-objective}
& \hyperref[app:cost]{\textbf{F}\quad Time and compute analysis}\dotfill\pageref{app:cost} \\
\hyperref[app:ablation]{\textbf{G}\quad Ablation studies}\dotfill\pageref{app:ablation}
& \\
\end{tabularx}
\end{minipage}
\vspace{2pt}

\section{Sampler-Side Batch Composition Procedure}
\label{app:sampler-algorithm}

Algorithm~\ref{alg:divopd-sampler} gives the reported procedure. After each
student turn, the frozen teacher scores the same response and the workflow
stores $s(u)$ from Equation~\ref{eq:learning-need}; the learner still receives
an ordinary OPD batch with unit per-turn weights.

$\textsc{Valid}(u)$ requires a complete prompt, tokens, teacher log
probabilities, and a nonempty action mask. When present, the teacher-valid mask
must align with and overlap the action mask; a missing mask is treated as
all-valid. Turns with inference or mask-alignment failures are discarded
because waiting cannot make their supervision usable.

\begin{algorithm}[H]
\caption{\raggedright\textbf{DivOPD sampler (\texttt{trajectory\_first} with
\texttt{turn\_signal=raw\_kl\_sum}).} Table~\ref{tab:sampler-config}
lists the reported numerical settings.}
\label{alg:divopd-sampler}
\small
\setlength{\tabcolsep}{0pt}
\renewcommand{\arraystretch}{1.08}
\begin{tabularx}{\linewidth}{@{}r@{\hspace{0.7em}}X@{}}
\toprule
& \textbf{Input:} update index $q$; pending turns $P$; rollout buffer
$\mathcal{E}$; batch size $B$; pool multiplier $c$; initial per-rollout cap
$k$; maximum staleness $\delta$; maximum pending age $W$.
\\
1 & Set $v_{\min}\gets\max(q-\delta,0)$ and remove every
$u\in P$ with $\operatorname{version}(u)<v_{\min}$. \\
2 & Read $F$ from $\mathcal{E}$ with version at least $v_{\min}$ until
$|P|+|F|=cB$ (or the read returns no more turns); set $C\gets P\cup F$. \\
3 & Set $V\gets\{u\in C:\textsc{Valid}(u)\}$ and permanently discard
$C\setminus V$. Stamp a fresh turn's first-seen update as $q$. \\
4 & Group $V$ by rollout identifier $r$. Order rollout groups by oldest
first-seen update, breaking ties by $r$. \\
5 & Within every group, sort turns by descending $s(u)$, breaking ties by
$(r,t)$. Initialize the selected set $S\gets\varnothing$ and cap $h\gets k$.
\\
6 & Sweep the ordered rollout groups once. From each group append its highest
ranked unselected turns until that group contributes $k$ turns or $|S|=B$.
\\
7 & \textbf{while} $|S|<B$ and an unselected turn remains in $V$:
set $h\gets h+1$ and sweep the groups again, adding the next-ranked turns
up to cap $h$. \\
8 & Set $P'\gets\{u\in V\setminus S:q-\operatorname{firstseen}(u)\le W\}$;
discard older unselected turns. \\
9 & \textbf{return} $S$ to the unchanged OPD learner
and retain $P'$ for update $q+1$. \\
\bottomrule
\end{tabularx}
\end{algorithm}

With $B=64$ and $k=4$, the first sweep covers at least
$\lceil B/k\rceil=16$ rollouts whenever the valid pool contains that many
groups. If the first sweep does not fill the batch, later passes raise the cap
and select more turns. The batch is short only when the pool contains fewer
than $B$ valid turns. Unselected valid turns remain pending for at most $W=8$
updates, while turns more than $\delta=2$
model versions old are removed before composition.

The reported configuration sets \texttt{weight\_mode=none} and
\texttt{task\_balance=false}. Within each rollout, selection is deterministic
given the scores and allocated slots. We do not use importance weights to
recover a uniform-turn objective. Selection changes the data-weighted
objective, while the per-turn OPD loss, reduction, and optimizer stay the same.

\begin{table}[!htbp]
\caption{\textbf{Shared sampler configuration.} All methods use the same
learner batch size. DivOPD-base and DivOPD share these sampler values; only
their within-rollout selection differs.}
\label{tab:sampler-config}
\centering
\small
\setlength{\tabcolsep}{6pt}
\lighttablestyle
\begin{tabularx}{\linewidth}{@{}lX@{}}
\toprule
\textbf{Setting} & \textbf{Value} \\
\midrule
Learner batch size $B$ & 64 turn rows \\
Candidate-pool multiplier $c$ & 4 (at most 256 pending and fresh turns) \\
Maximum policy staleness $\delta$ & 2 model versions \\
Maximum pending age $W$ & 8 learner updates \\
Initial per-rollout cap $k$
& ALFWorld 1.7B/4B: 4/4; WebShop 3B/7B: 5/4;
ScienceWorld 1.5B/3B: 4/5 \\
Per-turn loss settings & $\beta=1$; unit row weights \\
Task balancing & Disabled \\
\bottomrule
\end{tabularx}
\end{table}

\begin{table}[!htbp]
\caption{\textbf{Shared optimizer configuration.} These settings are fixed
across all methods.}
\label{tab:optimizer-config}
\centering
\small
\setlength{\tabcolsep}{6pt}
\lighttablestyle
\begin{tabularx}{\linewidth}{@{}lX@{}}
\toprule
\textbf{Setting} & \textbf{Value} \\
\midrule
Optimizer & AdamW \\
Learning rate & $10^{-6}$; constant schedule; no warmup \\
Adam moments & $\beta_1=0.9,\ \beta_2=0.999$ \\
Weight decay & $0.01$ \\
Gradient-norm clipping & $1.0$ \\
Optimizer batch & $N=64$ turn rows; seq-mean-token-mean
  (Equation~\ref{eq:opd-loss}) \\
\bottomrule
\end{tabularx}
\end{table}

\begin{table}[!htbp]
\caption{\textbf{Student initialization, frozen teachers, and evaluation protocol.}
Initial SR is measured before the first learner update under the reported
evaluation protocol; teacher SR uses the same fixed test episodes. Training responses are
sampled at temperature $1.0$, and evaluation uses temperature $0.4$.}
\label{tab:environment-config}
\centering
\small
\setlength{\tabcolsep}{5pt}
\lighttablestyle
\begin{tabularx}{\linewidth}{@{}lXr@{}}
\toprule
\textbf{Environment} & \textbf{Student} & \textbf{Initial SR (\%)} \\
\midrule
ALFWorld     & Qwen3-1.7B         &  7.86 \\
ALFWorld     & Qwen3-4B           & 26.43 \\
WebShop      & Qwen2.5-3B-Instruct &  1.56 \\
WebShop      & Qwen2.5-7B-Instruct & 14.06 \\
ScienceWorld & Qwen2.5-1.5B-Instruct & 21.48 \\
ScienceWorld & Qwen2.5-3B-Instruct & 45.31 \\
\bottomrule
\end{tabularx}

\vspace{5pt}

\begin{tabularx}{\linewidth}{@{}lXrr@{}}
\toprule
\textbf{Environment} & \textbf{Frozen teacher} & \textbf{Teacher SR (\%)} & \textbf{Episodes} \\
\midrule
ALFWorld & GiGPO-Qwen2.5-7B-Instruct-ALFWorld & 87.86 & 140 \\
WebShop & GiGPO-Qwen2.5-7B-Instruct-WebShop & 85.94 & 128 \\
ScienceWorld & Embodied-Planner-R1~\citep{fei2025unleashing} & 93.75 & 256 \\
\bottomrule
\end{tabularx}
\end{table}

\begin{table}[H]
\caption{\textbf{ScienceWorld task-type split.} Training and evaluation use
disjoint task types. Within each type, we retain the first half of the official
variation IDs before applying a fixed shuffle.}
\label{tab:scienceworld-split}
\centering
\footnotesize
\setlength{\tabcolsep}{5pt}
\renewcommand{\arraystretch}{1.02}
\lighttablestyle
\begin{tabularx}{\linewidth}{@{}>{\raggedright\arraybackslash}X>{\raggedright\arraybackslash}X@{}}
\toprule
\textbf{Training task types (17)} & \textbf{Held-out evaluation task types (13)} \\
\midrule
boil; melt; change-the-state-of-matter-of; use-thermometer;
measure-melting-point-known-substance; power-component; test-conductivity;
find-living-thing; find-plant; grow-plant; chemistry-mix;
chemistry-mix-paint-secondary-color; lifespan-shortest-lived;
identify-life-stages-2; inclined-plane-determine-angle;
inclined-plane-friction-named-surfaces; mendelian-genetics-known-plant
&
freeze; measure-melting-point-unknown-substance;
power-component-renewable-vs-nonrenewable-energy;
test-conductivity-of-unknown-substances; find-non-living-thing; find-animal;
grow-fruit; chemistry-mix-paint-tertiary-color; lifespan-longest-lived;
lifespan-longest-lived-then-shortest-lived; identify-life-stages-1;
inclined-plane-friction-unnamed-surfaces; mendelian-genetics-unknown-plant
\\
\bottomrule
\end{tabularx}
\end{table}

\textbf{ScienceWorld subset construction.}
For a task type with $N$ official variations, we retain variation IDs
$0,\ldots,\lfloor N/2\rfloor-1$ and shuffle each split once with seed 42. This
produces 2,294 training instances and 1,308 instances from held-out task
types. Training cycles sequentially through the fixed shuffled training file.
Every evaluation uses the same first 256 instances of the fixed shuffled
held-out file; the frozen-teacher and initial-student SRs use these episodes as
well. Thus, no ScienceWorld task type is shared between training and
evaluation.

Evaluation processes each test set in order every five explorer steps
(every 20 for the ALFWorld focus-only runs), giving 24--49 evaluations per run.
The maximum interaction horizon is 30 turns on ALFWorld and ScienceWorld and 15
on WebShop. Runs are configured for 250 trainer updates on ALFWorld and
ScienceWorld and 150 on WebShop; exceptions are focus-only on ScienceWorld 3B
(150), WebShop 3B (130), and WebShop 7B (111), and DivOPD-base (138), DivOPD
(134), and TurnOPD (132) on WebShop 7B.
The matched-access cap comparison in Table~\ref{tab:matched-access-cap}
uses seed 42 for both arms in each setting.

\subsection{Reference variant and comparison scope}
\label{sec:variants}

DivOPD-base applies the validity gate and rollout-first coverage with uniform
within-rollout sampling. DivOPD uses the same allocation rules and adds focus;
DivOPD$+$R additionally changes rollout generation through recovery.
These variants share the numerical sampler settings in
Table~\ref{tab:sampler-config}, unit per-turn weights, and no task balancing.

DivOPD-focus is an earlier implementation. It has no validity gate, visits
task groups in seeded random order, ranks turns by disagreement, caps a group
at three turns and a task's token share at $25\%$ on the first pass, removes
the turn cap when relaxing, and enables per-rollout task-balance loss weights.
Some runs are shorter and the ALFWorld evaluation frequency differs, as noted
above. Its comparison with DivOPD is therefore descriptive.

Focus-only has the highest per-batch effective rollout count in five settings,
while DivOPD has higher peak SR in five
(Table~\ref{tab:coverage-cost-full}). DivOPD with $k=3$ also exceeds focus-only
in peak SR in all six settings (mean $+4.9$ points), but gating, ordering,
relaxation, and weighting differ. These comparisons do not isolate the cap
or focus; the matched controls in Section~\ref{sec:ablation-main} do.

\section{Fixed-Budget Efficiency}
\label{app:fixed-budget}

We reconstruct evaluation curves from the archived explorer logs, map each
evaluation to its learner policy version through \texttt{rollout/model\_version},
and linearly interpolate the unsmoothed success rates. Explorer log steps are
not learner updates and are not used as the budget axis. For each setting,
$U$ is the smallest final evaluated learner version among vanilla OPD,
TCOD-F2B, TurnOPD, DivOPD-base, DivOPD, and DivOPD$+$R. The normalized area
is $\mathrm{nAUC}=U^{-1}\int_0^U\mathrm{SR}(q)\,dq$ over these interpolated
curves; Figure~\ref{fig:fixed-budget} (Section~\ref{sec:efficiency}) plots four
of the methods over the shared budget.

The aligned audit is reproduced by \texttt{plot\_final\_narrative.py} from the
supplied logs.

\section{Full Per-Setting Main Results}
\label{app:main-results-full}

Appendix Table~\ref{tab:main-results-full} expands Table~\ref{tab:main-results}
with update counts and absolute trainer GPU-hours for every method and setting.

\fullmainresultstable

\section{Full Batch-Composition Diagnostics}
\label{app:batch-diagnostics}

\begin{figure}[ht]
\centering
\includegraphics[width=0.94\linewidth]{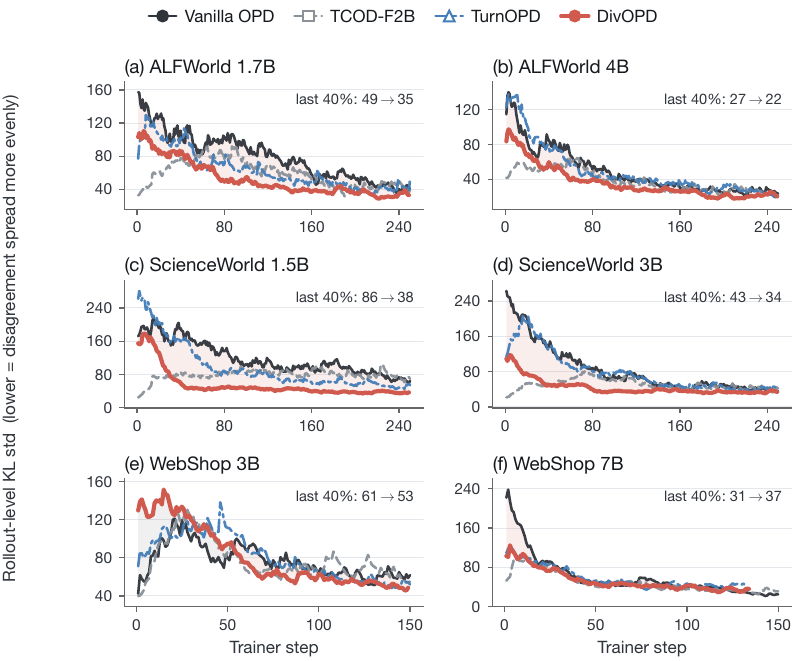}
\caption{\textbf{Cross-rollout disagreement in all six settings} (main-text
Figure~\ref{fig:kl-stability} shows panels a, c, and e). Curves are the standard deviation of
rollout-summed student-minus-teacher log probabilities (EMA, $\alpha=0.12$); each
panel gives the mean over the last $40\%$ of training (vanilla $\to$ DivOPD). Red
shading marks where DivOPD is lower and gray where it is higher. Variation falls
on ALFWorld and ScienceWorld.}
\label{fig:kl-stability-all}
\end{figure}

\begin{table}[!htbp]
\caption{\textbf{What DivOPD changes.} The core method changes batch selection;
only the optional recovery extension changes rollout generation.}
\label{tab:layers}
\centering
\small
\lighttablestyle
\begin{tabularx}{\linewidth}{@{}>{\raggedright\arraybackslash}p{0.40\linewidth}
  >{\raggedright\arraybackslash}p{0.20\linewidth}
  >{\raggedright\arraybackslash}X@{}}
\toprule
\textbf{Design question} & \textbf{Component} & \textbf{DivOPD} \\
\midrule
How is each turn scored? & Learning signal & Unchanged (Eq.~\ref{eq:opd-adv}) \\
How long can a rollout continue? & Rollout horizon & Unchanged \\
Who acts during the rollout? & Rollout policy & Changed only by DivOPD$+$R (\S\ref{sec:recovery}) \\
Which turns enter a learner batch? & Batch selection & Changed by DivOPD (\S\ref{sec:method}) \\
\bottomrule
\end{tabularx}
\end{table}

Table~\ref{tab:layers} distinguishes batch selection from changes to the
learning signal or rollout policy. Agent-RL methods such as Text2Grad, GiGPO,
BiPACE, and SkillRL change the feedback extracted from interaction
\citep{wang2026text2grad,feng2026group,wang2026bipace,xia2026skillrl}.
Table~\ref{tab:coverage-cost-full} reports the complete six-setting audit behind
the diagnostics summarized in Figure~\ref{fig:clustering}; the online
cross-rollout disagreement diagnostic appears in Figure~\ref{fig:kl-stability}.

\clearpage
\begin{table}[H]
\caption{\textbf{Full batch-composition diagnostics.} Values are averaged over
the final $40\%$ of training; peak SR is repeated from Appendix
Table~\ref{tab:main-results-full}. Focus-only maximizes coverage in five settings,
while DivOPD has the higher peak SR in five. Bold marks the best peak SR among
the five methods shown, which exclude DivOPD$+$R. Figure~\ref{fig:clustering}
averages the first three diagnostics over the two student sizes per environment.}
\label{tab:coverage-cost-full}
\centering
\scriptsize
\setlength{\tabcolsep}{3.6pt}
\renewcommand{\arraystretch}{1.00}
\begin{tabular}{@{}llrrrrrr@{}}
\toprule
\tabhead
\textbf{Setting} & \textbf{Method}
 & \textbf{Eff.\ rollouts} & \textbf{Top-1 tok.}
 & \textbf{Turns per} & \textbf{KL / token}
 & \textbf{Step time} & \textbf{Peak SR} \\
\tabhead
& & \textbf{/ batch}\,$\uparrow$ & \textbf{share}\,$\downarrow$ & \textbf{rollout}
 & $\downarrow$ & (s)\,$\downarrow$ & (\%)\,$\uparrow$ \\
\midrule
ALFWorld 4B
 & Vanilla OPD  &  6.7 & 0.281 &  8.1 & 0.121 & 94.9 & 86.43 \\
& TCOD-F2B     &  6.5 & 0.291 &  8.6 & 0.114 & 99.6 & 87.14 \\
& DivOPD-focus & 18.6 & 0.110 &  2.4 & 0.174 & 85.2 & 89.29 \\
& DivOPD-base  & 12.5 & 0.136 &  3.9 & 0.111 & 97.8 & \bst{92.14} \\
& \textbf{DivOPD} & 12.5 & 0.145 &  3.9 & 0.176 & 82.1 & 91.43 \\
\addlinespace[3pt]
ALFWorld 1.7B
 & Vanilla OPD  &  5.2 & 0.364 & 10.2 & 0.222 & 51.2 & 77.86 \\
& TCOD-F2B     &  5.4 & 0.346 & 10.0 & 0.200 & 49.0 & 72.86 \\
& DivOPD-focus & 16.1 & 0.120 &  3.0 & 0.280 & 40.3 & \bst{80.00} \\
& DivOPD-base  & 12.4 & 0.132 &  3.9 & 0.201 & 44.4 & 74.29 \\
& \textbf{DivOPD} & 12.4 & 0.140 &  3.9 & 0.289 & 39.5 & 79.29 \\
\addlinespace[3pt]
WebShop 3B
 & Vanilla OPD  &  4.7 & 0.334 &  7.8 & 0.179 & 106.3 & 60.94 \\
& TCOD-F2B     &  5.0 & 0.328 &  7.4 & 0.167 & 101.9 & 67.19 \\
& DivOPD-focus & 25.7 & 0.069 &  2.0 & 0.181 & 64.6 & 70.31 \\
& DivOPD-base  & 12.4 & 0.114 &  4.7 & 0.151 & 76.7 & \bst{78.12} \\
& \textbf{DivOPD} & 12.3 & 0.114 &  4.7 & 0.156 & 70.9 & \bst{78.12} \\
\addlinespace[3pt]
WebShop 7B
 & Vanilla OPD  &  8.3 & 0.237 &  6.1 & 0.035 & 126.7 & 87.50 \\
& TCOD-F2B     &  7.5 & 0.271 &  6.2 & 0.029 & 120.1 & 89.06 \\
& DivOPD-focus & 29.2 & 0.067 &  1.5 & 0.052 & 90.2 & 89.84 \\
& DivOPD-base  & 14.8 & 0.101 &  3.9 & 0.027 & 104.5 & 88.28 \\
& \textbf{DivOPD} & 15.0 & 0.096 &  3.9 & 0.036 & 96.3 & \bst{90.62} \\
\addlinespace[3pt]
ScienceWorld 3B
 & Vanilla OPD  &  6.2 & 0.272 &  9.5 & 0.241 & 18.9 & 85.55 \\
& TCOD-F2B     &  6.2 & 0.253 &  9.3 & 0.242 & 19.3 & 87.89 \\
& DivOPD-focus & 14.6 & 0.115 &  3.9 & 0.487 & 22.7 & 77.34 \\
& DivOPD-base  & 11.3 & 0.146 &  4.7 & 0.257 & 21.1 & 86.72 \\
& \textbf{DivOPD} & 11.8 & 0.137 &  4.7 & 0.336 & 19.1 & \bst{88.28} \\
\addlinespace[3pt]
ScienceWorld 1.5B
 & Vanilla OPD  &  3.9 & 0.400 & 12.9 & 0.349 & 15.2 & 66.02 \\
& TCOD-F2B     &  4.1 & 0.383 & 12.5 & 0.341 & 14.2 & 71.88 \\
& DivOPD-focus & 12.8 & 0.121 &  4.3 & 0.513 & 14.1 & 66.80 \\
& DivOPD-base  & 12.1 & 0.130 &  4.4 & 0.341 & 16.3 & 70.31 \\
& \textbf{DivOPD} & 13.5 & 0.115 &  4.0 & 0.483 & 13.8 & \bst{78.91} \\
\bottomrule
\end{tabular}
\end{table}

\section{Executed OPD Objective}
\label{app:opd-objective}

Vanilla OPD, DivOPD-base, DivOPD, and the matched-cap controls use unit row
weights with the same seq-mean-token-mean reduction. Its scalar value is
\begin{equation}
\mathcal{L}_{\mathrm{OPD}}^{\mathrm{value}}
= \frac{1}{N}\sum_{t=1}^{N}\ell_t,
\qquad
\ell_t =
\begin{cases}
-\dfrac{1}{|V_t|}\displaystyle\sum_{i\in V_t} A_{t,i}, & |V_t|>0,\\[4pt]
0, & |V_t|=0.
\end{cases}
\label{eq:opd-loss}
\end{equation}
\paragraph{Stored scores and recomputed learner probabilities.}
The workflow stores student log probabilities at the generating policy
$\theta_{v(u)}$ and frozen-teacher log probabilities on the same response.
Equation~\ref{eq:learning-need} uses these cached values for selection.
Before update $q$, the learner recomputes student log probabilities at its
current snapshot $\bar\theta_q$, which is $\theta_{\mathrm{old}}$ in
Equation~\ref{eq:opd-adv}; teacher scores remain fixed.
The PPO surrogate uses
\[
r_{t,i}(\theta)=\exp\!\left[
\log\pi_\theta(y_{t,i}\mid h_{t,<i})
-\operatorname{sg}\!\left[\log\pi_{\bar\theta_q}(y_{t,i}\mid h_{t,<i})\right]
\right].
\]
At gradient evaluation, $\theta=\bar\theta_q$, so $r_{t,i}=1$ and clipping is
inactive, but $\nabla_\theta r_{t,i}$ is not zero. Each batch receives one
optimizer update. There is no additional behavior-to-learner importance
correction: bounded staleness limits the sampling mismatch but does not
remove it. Vanilla OPD and the controlled composer variants share this
update rule. The outer average includes dead rows, and the gradient is
\begin{equation}
\nabla_\theta \mathcal{L}_{\mathrm{OPD}}
= \frac{1}{N}\sum_{t:\,|V_t|>0}\frac{1}{|V_t|}
\sum_{i\in V_t}-A_{t,i}\,
\nabla_\theta\log\pi_{\theta}(y_{t,i}\mid h_{t,<i}).
\label{eq:opd-gradient}
\end{equation}
Thus valid tokens are averaged within each turn row, rows are equally weighted,
and a dead row contributes zero while remaining in the $N$-row denominator.
The formula describes the executed update on selected, potentially lagged
samples; it is not an unbiased current-policy reverse-KL gradient.

\subsection{Gradient effect of batch composition}
\label{app:removal}
\begingroup
\small

Because Equation~\ref{eq:opd-gradient} is a plain average of per-row terms,
its selection-induced change can be written exactly for a fixed pool.
These equations describe reweighting, not a guarantee of better policy
performance. Let the candidate pool hold
$N_c$ turns and write the contribution of turn $t$ as
$g_t=\frac{1}{|V_t|}\sum_{i\in V_t}-A_{t,i}\nabla_\theta\log\pi_\theta(y_{t,i}\mid h_{t,<i})$,
with $g_t=0$ for a dead row. The pool gradient is $G=\frac{1}{N_c}\sum_t g_t$.
A composer selects $B$ rows through indicators $m_t\in\{0,1\}$ with
$\sum_t m_t=B$; let $\pi_t=\mathbb{E}[m_t]$ be the inclusion probability, so
$\sum_t\pi_t=B$. The optimizer sees $\hat G=\frac{1}{B}\sum_t m_t g_t$.

\begin{lemma}[Selection-induced shift identity]
\label{lem:bias}
For any full composer batch, $\mathbb{E}[\hat G]-G
=\frac{N_c}{B}\,\mathrm{Cov}_t(\pi_t,g_t)$, where this is the scalar--vector
covariance over a uniformly drawn pool row.
\end{lemma}
\begin{proof}
$\mathbb{E}[\hat G]=\frac{1}{B}\sum_t\pi_t g_t=\frac{N_c}{B}\,\mathbb{E}_t[\pi_t g_t]$.
Since $\mathbb{E}_t[\pi_t]=B/N_c$, also $G=\mathbb{E}_t[g_t]=\frac{N_c}{B}\,\mathbb{E}_t[\pi_t]\,\mathbb{E}_t[g_t]$.
Subtracting gives $\frac{N_c}{B}\big(\mathbb{E}_t[\pi_t g_t]-\mathbb{E}_t[\pi_t]\mathbb{E}_t[g_t]\big)$.
\end{proof}

\paragraph{Score--loss relation at a common policy version.}
At the same student policy version $v$, the common valid mask gives
\begin{equation}
s_v(u)=-\frac{1}{\beta}\sum_{i\in V_u}A^{(v)}_{u,i}
=\frac{|V_u|}{\beta}\,\ell_u^{\mathrm{value},(v)}.
\label{eq:score-loss-equivalence}
\end{equation}
With a complete response mask, its conditional expectation is the sequence-level
reverse KL, $D_{\mathrm{KL}}(P_{\theta_v}(\cdot\mid h)\|P_\phi(\cdot\mid h))$.
Thus raw $s_v$ ranks token count times row-mean OPD loss, while $s_v/|V_u|$
ranks the row loss. Neither is the gradient norm because score-gradient vectors
can vary and cancel. Cached scores use rollout version $v$; learner-side
recomputation can be up to $\delta$ versions later. The equality therefore
does not identify the cached score with the row loss evaluated at the later
learner snapshot.

The identity shows how each stage changes the expected gradient. The gate
removes rows with $g_t=0$ in the audited buffers. If the remaining rows are
sampled uniformly, this increases the expected gradient magnitude without
changing its direction, as shown below. Coverage limits how much rollout
length affects $\pi_t$, moving from turn-uniform toward rollout-uniform
averaging. Focus then favors turns with larger $s(u)$ within each selected
rollout. Section~\ref{app:geometry} measures the resulting gradient changes.

The following sampling calculation isolates dead-slot dilution under
uniform sampling. Its variance result uses simplifying assumptions: rollout-correlated
gradients and deterministic selection need not satisfy its independence
assumptions.

\begin{proposition}[Dead-slot dilution under uniform sampling]
\label{prop:gate-dilution}
If a pool contains $D$ dead rows and the gate samples uniformly from its
$N_c-D$ live rows, then
$\mathbb{E}[\hat G_{\mathrm{gate}}]=\frac{N_c}{N_c-D}G$: the gate preserves
direction while undoing dead-slot shrinkage. Further, suppose live-row
gradients are i.i.d. with mean $\mu$ and covariance $\Sigma$, and independent
slot validity is $Z\sim\mathrm{Bernoulli}(1-d)$, where $d=D/N_c$ is the
dead-row fraction. For the same valid-row target,
\begin{equation}
\mathrm{Cov}(\hat\mu_{\mathrm{gate}})=\frac{\Sigma}{B},\qquad
\mathrm{Cov}\!\left(\frac{\hat G_{\mathrm{ungated}}}{1-d}\right)
=\frac{\Sigma+d\mu\mu^\top}{B(1-d)}.
\label{eq:gate-variance}
\end{equation}
Hence gating is a positive-semidefinite variance reduction; at $\mu=0$, its
variance is exactly $(1-d)$ times the validity-corrected ungated variance.
\end{proposition}
\begin{proof}
Dead rows have $g_t=0$, so uniform live-row sampling gives
$\mathbb{E}[\hat G_{\mathrm{gate}}]=(N_c-D)^{-1}\sum_{\mathrm{live}}g_t
=N_cG/(N_c-D)$. For the stochastic statement, let $X$ denote a live-row
gradient. Then $\mathbb{E}[ZX]=(1-d)\mu$ and
$\mathrm{Cov}(ZX)=(1-d)\Sigma+d(1-d)\mu\mu^\top$. Independence across the $B$
slots yields Equation~\ref{eq:gate-variance}; subtracting $\Sigma/B$ leaves
$d(\Sigma+\mu\mu^\top)/[B(1-d)]\succeq0$.
\end{proof}
At $\mu=0$, validity-corrected ungated reading is variance-equivalent to a live-only batch
of effective size $B_{\mathrm{eff}}=(1-d)B$; for $B=64$ and
$d=0.05$--$0.28$, this is $46$--$61$ rows, whereas the gate restores all $64$.
The scale factor $1/(1-d)$ is largely suppressed before optimization in our
runs: gradient norms are clipped at $1.0$ and logged pre-clip norms are
$2$--$15$. Under the simplifying assumption of clipped SGD, the first-order comparison then
reduces to the direction of $\hat G$; AdamW's history-dependent
preconditioning prevents this from being an exact equivalence. We therefore
interpret $B_{\mathrm{eff}}$ as a sampling-efficiency statement rather than an
exact claim about optimizer step length. Empirically, gate-only
preserves mean $s(u)$ per valid ungated turn in all four replays; on the ALFWorld 1.7B DivOPD buffer,
$14.59/(1-0.1471)=17.11$. Focus is evaluated separately through the
gradient-geometry audit and downstream results.

\begin{corollary}[Coverage as first-sweep capped rollout averaging]
\label{cor:cover}
Condition on the oldest-first selected rollout set $\mathcal R_S$ and allocations
$\{n_r\}$, where $\sum_r n_r=B$. Uniform within-rollout sampling gives
\begin{equation}
\mathbb{E}[\hat G_{\mathrm{base}}\mid\mathcal R_S,\{n_r\}]
=\sum_{r\in\mathcal R_S}\frac{n_r}{B}\,\bar g_r.
\label{eq:coverage-estimand}
\end{equation}
When the batch fills in the first sweep, $n_r=\min(k,T_r)$ for each fully
allocated group; the final group may receive only the remaining slots.
If all selected $T_r\geq k$ and $B$ is divisible by $k$, the result is the
rollout-uniform mean over selected groups. If every rollout is selected in
full, it is the original turn-uniform mean.
\end{corollary}
The first sweep therefore moves between turn-uniform and rollout-uniform
averaging by capping length weights. Holding the selected set fixed separates
that reweighting from oldest-first rollout selection and gives an exact
description of the coverage stage.

\begin{corollary}[Focus reweighting identity]
\label{cor:focus}
For the same $\mathcal R_S$ and $\{n_r\}$, let $K_r$ be the deterministic
top-$n_r$ set by $s(u)$ and $D_r$ its complement. Then
\begin{equation}
\hat G_{\mathrm{focus}}
-\mathbb{E}[\hat G_{\mathrm{base}}\mid\mathcal R_S,\{n_r\}]
=\frac{1}{B}\sum_{r:n_r<T_r}
\frac{n_r(T_r-n_r)}{T_r}
\left(\bar g_r^{K}-\bar g_r^{D}\right).
\label{eq:focus-estimand}
\end{equation}
This follows from
$\bar g_r=(n_r/T_r)\bar g_r^K+((T_r-n_r)/T_r)\bar g_r^D$.
\end{corollary}
Thus the change caused by focus depends on the difference between the mean
gradients of kept and dropped turns. Table~\ref{tab:dropped} also compares
their gradient norms: kept turns account for a larger share of total gradient
norm than of turn count, and tend to be earlier, higher-disagreement decisions.
\endgroup

\begin{table}[ht]
\caption{\textbf{Focus retention audit} ($k=5$, the initial cap used in
Table~\ref{tab:main-results} for these settings). Dropped gradient norm is the
share of summed per-turn norms, $\sum_t\|g_t\|$, estimated by the sketch audit.
The last four columns are means for kept and dropped turns.}
\label{tab:dropped}
\centering
\small
\setlength{\tabcolsep}{5pt}
\lighttablestyle
\begin{tabular*}{\linewidth}{@{\extracolsep{\fill}}lrrrrrr@{}}
\toprule
& \multicolumn{2}{c}{\textbf{Dropped share}}
& \multicolumn{2}{c}{\textbf{Turn depth}} & \multicolumn{2}{c}{\textbf{Score/token}} \\
\cmidrule(lr){2-3}\cmidrule(lr){4-5}\cmidrule(lr){6-7}
\textbf{Setting} & Turns & Grad.\ norm & Kept & Dropped & Kept & Dropped \\
\midrule
ScienceWorld 3B & 55\% & 40\% & 4.7 & 9.5 & 0.30 & 0.16 \\
WebShop 3B      & 18\% & 12\% & 2.3 & 4.5 & 0.16 & 0.08 \\
\bottomrule
\end{tabular*}
\end{table}

\subsection{Rollout utilization under a staleness budget}
\label{app:utilization}

\begin{table}[ht]
\caption{\textbf{Cumulative rollout utilization.}
Never selected is the share of valid rollouts contributing no trained turn;
Gini measures concentration of selected-turn counts. ALF/SciW/WS denote
ALFWorld/ScienceWorld/WebShop. The separate-buffer result is grouped below
the shared-buffer comparisons and is not a matched control.}
\label{tab:utilization-full}
\centering\small
\setlength{\tabcolsep}{5pt}
\lighttablestyle
\begin{tabular*}{\linewidth}{@{\extracolsep{\fill}}lrrrrrr@{}}
\toprule
& \multicolumn{3}{c}{\textbf{Never selected (\%)}} & \multicolumn{3}{c}{\textbf{Gini}} \\
\cmidrule(lr){2-4}\cmidrule(lr){5-7}
\textbf{Composer} & ALF & SciW & WS & ALF & SciW & WS \\
\midrule
\multicolumn{7}{l}{\textit{Shared buffers: first 100 updates}} \\
Vanilla OPD (arrival-order) & 53.9 & 49.2 & 58.7 & 0.66 & 0.68 & 0.64 \\
Top-$B$ (no rollout limit)  & --   & --   & --   & 0.44 & 0.45 & 0.27 \\
DivOPD-focus                & 15.0 & 10.2 &  3.4 & 0.33 & 0.32 & 0.20 \\
DivOPD-base / DivOPD        & 14.2 & 11.0 &  7.1 & 0.30 & 0.32 & 0.12 \\
\midrule
\multicolumn{7}{l}{\textit{Separate buffers: full 250 updates}} \\
DivOPD, $k=\infty$ & 56.6 & 67.0 & 39.9 & 0.70 & 0.79 & 0.51 \\
\bottomrule
\end{tabular*}
\end{table}

Top-$B$ ranks all valid turns by $s(u)$ without a per-rollout limit.
DivOPD-base and DivOPD coincide because focus does not change rollout allocation.
The separate 250-update buffers were generated by $k=\infty$ runs; this replay
is a qualitative access diagnostic, not the uniform-within-rollout,
matched-access training control in Table~\ref{tab:matched-access-cap}.

Consider a group of $R$ rollouts that remains eligible for $L$ updates.
If each update touches at most $M$ distinct rollouts, at most
$\min(R,LM)$ members of that group can be selected. This is a group-level
upper bound, not a steady-state utilization estimate: other arrival cohorts
compete for the same slots, and a rollout can be selected repeatedly.
An arrival-order reader consuming consecutive rollout blocks touches roughly
$B/\bar T$ rollouts per update, where $\bar T$ is mean rollout length.
The first-pass cap instead covers at least $\lceil B/k\rceil$ groups when
enough valid groups are available. It therefore increases opportunities for
rollout participation, while actual lifetime coverage must be measured across
updates. Table~\ref{tab:utilization-full} provides that measurement. Every full
composer batch still contains $B=64$ rows; the cap redistributes them across
rollouts rather than increasing the row budget.

\paragraph{Effect of a wider staleness window.}
Relative to the main vanilla references with
$\mathrm{max\_staleness}=2$, two auxiliary runs with a limit of 4 leave
row-level consumption unchanged and consume rows $1.3$--$1.5$ versions older
(Table~\ref{tab:staleness}). Peak SR is lower by $6.4$ points on ALFWorld and
$0.4$ points on ScienceWorld, while the mean current-versus-rollout
probability gap is $10$--$15\%$ larger.

\begin{table}[t]
\caption{\textbf{Vanilla OPD with a wider staleness window.} Mean age of
consumed rows in policy versions (arrival-order replay, dead rows included),
peak and final-checkpoint SR (\%). The $\mathrm{max\_staleness}=2$ rows reuse the main
vanilla runs in Table~\ref{tab:main-results}; the wider-window rows are
separate auxiliary runs rather than paired reruns.}
\label{tab:staleness}
\centering\small
\lighttablestyle
\begin{tabular*}{\linewidth}{@{\extracolsep{\fill}}llrrr@{}}
\toprule
\textbf{Setting} & $\mathrm{max\_staleness}$ & \textbf{Mean age} & \textbf{Peak} & \textbf{Final ckpt.} \\
\midrule
ALFWorld 1.7B     & 2 & 1.69 & 77.86 & 73.57 \\
                  & 4 & 3.19 & 71.43 & 71.43 \\
ScienceWorld 1.5B & 2 & 1.68 & 66.02 & 66.02 \\
                  & 4 & 2.96 & 65.62 & 60.16 \\
\bottomrule
\end{tabular*}
\end{table}

\paragraph{Sensitivity to bounded policy lag.}
The usual reverse-KL gradient identity is exact when samples are drawn from
the current learner policy. In our asynchronous implementation, trajectories
may instead be generated by a behavior policy up to two versions behind the
learner, so the executed update is a bounded-lag approximation. Across the six
main settings, rollout and learner action probabilities remain highly
correlated ($0.956$--$0.989$), with mean absolute differences of only
$0.011$--$0.015$. As an additional offline stress test, recomputing DivOPD
scores at later checkpoints gives cached--current Spearman correlations of
$0.984$ on ScienceWorld and $0.997$ on WebShop; the corresponding
within-trajectory top-$k$ selection overlaps are $93.5\%$ ($k=4$) and
$99.6\%$ ($k=5$). In these tests, recomputing the scores within the
two-version lag window rarely changes which turns are selected. This supports
cached ranking in the tested setting, not an unbiased-gradient claim.

\subsection{Gradient diagnostics}
\label{app:geometry}

We replay every composer on the same candidate pools over the last eight
updates of each buffer. To compare per-turn gradients without storing them
in full, we use a Kronecker projection with unbiased inner products.
Table~\ref{tab:geometry} reports $\langle G,\hat G\rangle/\|G\|^2$ and
$\cos(G,\hat G)$, where $G$ is the pool gradient and $\hat G$ is the selected
batch gradient. Focus raises the first metric over uniform within-rollout
sampling in all 24 audited updates; its mean cosine is also higher in all
three environments.

The proxy measures alignment with the specified candidate pool, not an
optimal policy-improvement direction. It is a pre-optimizer diagnostic;
training outcomes are tested separately in Table~\ref{tab:main-results}.
Mean pairwise gradient cosine is $0.003$--$0.02$ for every composer; reduced
gradient correlation is therefore not the main change observed here.

Selecting turns by alignment in the same sketch space gives an optimistic
reference, with a proxy $1.3$--$2.1\times$ that of the best $s(u)$ rule in each
environment. Using a
rollout-mean pool gradient leaves the rankings unchanged. All these comparisons
use raw gradients, before optimizer preconditioning.

\begin{table}[ht]
\caption{\textbf{Selected gradients versus the candidate-pool gradient.}
Means over 8 audited updates per environment. Projection is
$\langle G,\hat G\rangle/\|G\|^2$; cosine similarity is $\cos(G,\hat G)$.
ALF/SciW/WS denote ALFWorld/ScienceWorld/WebShop.
$^\dagger$Arrival order here draws from the
finite-$k$ composer's pool and is not vanilla OPD. Sketch oracle selects and
evaluates the $B$ turns with the largest sketched $\langle G,g_t\rangle$, so it
is an optimistic reference rather than a full-gradient oracle.}
\label{tab:geometry}
\centering\small
\setlength{\tabcolsep}{4.5pt}
\lighttablestyle
\begin{tabular*}{\linewidth}{@{\extracolsep{\fill}}lrrrrrr@{}}
\toprule
& \multicolumn{3}{c}{\textbf{Gradient projection}} & \multicolumn{3}{c}{\textbf{Cosine similarity}} \\
\cmidrule(lr){2-4}\cmidrule(lr){5-7}
\textbf{Composer} & ALF & SciW & WS & ALF & SciW & WS \\
\midrule
Arrival order$^\dagger$ & 1.74 & 0.75 & 1.32 & 0.97 & 0.49 & 0.61 \\
Gate only                & 1.11 & 1.02 & 1.29 & 0.93 & 0.53 & 0.62 \\
DivOPD-base             & 1.36 & 0.97 & 1.31 & 0.96 & 0.54 & 0.62 \\
DivOPD                  & 2.04 & 2.01 & 1.47 & 0.98 & 0.75 & 0.64 \\
DivOPD-focus             & 3.12 & 1.66 & 1.50 & 0.99 & 0.67 & 0.65 \\
Top-$B$ ($k=\infty$)    & 2.97 & 2.30 & 1.61 & 0.99 & 0.79 & 0.68 \\
Sketch oracle            & 4.19 & 3.70 & 3.30 & 1.00 & 0.91 & 0.84 \\
\bottomrule
\end{tabular*}
\end{table}

\subsection{Which turns each stage selects}
\label{app:replay-audit}

Table~\ref{tab:replay-depth-outcome} follows the ALFWorld 1.7B DivOPD ($k=4$) buffer through the
selection chain. ``Vanilla'' reconstructs the arrival-order
reader (one batch per update in queue order with the staleness filter), including
invalid slots; ``Gate'' keeps that order but takes the first $B$ valid rows; ``Cover'' and ``Focus'' add the remaining
stages. Queue row ID provides a consistent arrival-order proxy for this
controlled replay.

\begin{table}[!htbp]
\caption{\textbf{Where selected turns come from in the ALFWorld 1.7B DivOPD replay.}
Share of selected turns (\%), grouped by turn depth and rollout outcome.
Each method column sums to approximately $100\%$ after rounding;
stages are added one at a time from left to right.}
\label{tab:replay-depth-outcome}
\centering\small
\setlength{\tabcolsep}{5.0pt}
\lighttablestyle
\begin{tabular*}{\linewidth}{@{\extracolsep{\fill}}llrrrr@{}}
\toprule
\textbf{Depth} & \textbf{Outcome} & \textbf{Vanilla} & \textbf{$+$Gate} & \textbf{$+$Coverage} & \textbf{$+$Focus} \\
\midrule
\multirow{2}{*}{$<2$} & Failure & 3.2 & 3.8 & 6.1 & 10.6 \\
& Success & 17.6 & 20.7 & 14.3 & 15.8 \\
\addlinespace[2pt]
\multirow{2}{*}{$[2,4)$} & Failure & 3.1 & 3.6 & 6.1 & 8.5 \\
& Success & 17.4 & 20.4 & 14.1 & 14.5 \\
\addlinespace[2pt]
\multirow{2}{*}{$[4,6)$} & Failure & 2.9 & 3.4 & 5.9 & 7.4 \\
& Success & 13.1 & 15.3 & 9.7 & 9.7 \\
\addlinespace[2pt]
\multirow{2}{*}{$[6,10)$} & Failure & 5.6 & 6.6 & 11.9 & 11.9 \\
& Success & 11.4 & 13.4 & 8.5 & 7.2 \\
\addlinespace[2pt]
\multirow{2}{*}{$\geq10$} & Failure & 23.1 & 9.9 & 20.7 & 12.4 \\
& Success & 2.6 & 3.0 & 2.7 & 2.1 \\
\bottomrule
\end{tabular*}
\end{table}

The gate removes unscored turns late in failed rollouts. Coverage retains
other valid turns from those rollouts for later updates, while focus favors
earlier, higher-disagreement turns within them. Coverage and focus select the
same fraction of turns from successful rollouts ($49.3\%$), so their difference
is not explained by selecting more successes. The
arrival-order reader, meanwhile, consumes $62.1\%$ successful turns against
$22.8\%$ in the pool. This is consistent with shorter successful rollouts
finishing and entering the queue earlier. On vanilla OPD's own ALFWorld 1.7B buffer, the same
reader selects $48.1\%$ successful turns, and its $19.5\%$ dead-slot fraction
matches the online trainer log.

\FloatBarrier
\section{Where the Time Goes}
\label{app:cost}

We measure trained tokens, learner GPU time, and elapsed trainer-loop time
separately. The first two quantify learner-side efficiency; the third captures
elapsed time including experience reads, not total system GPU consumption.
Teacher scoring is performed for generated turns whether or not they are
selected. Recovery adds teacher generation, which is excluded from the
reported $+$R learner-GPU speedup. These measurements do not provide a
full accounting of explorer and teacher GPU time or queue memory.

\begin{table}[!htbp]
\caption{\textbf{Resource efficiency to $\tau$.} Vanilla-to-DivOPD ratios;
larger is better. Trainer GPU time uses the two GPUs assigned to the learner.
The wall estimate sums trainer-step and experience-read time. Its speedup differs
from timestamp-based end-to-end measurement by at most $1.2\%$ on the three
settings whose timestamps were retained.}
\label{tab:resource-efficiency}
\centering\small
\setlength{\tabcolsep}{4.0pt}
\lighttablestyle
\begin{tabular*}{\linewidth}{@{\extracolsep{\fill}}lrrr@{}}
\toprule
\textbf{Setting} & \textbf{Trained tokens} & \textbf{Trainer GPU} & \textbf{Wall time (est.)} \\
\midrule
ALFWorld 4B       & 1.92$\times$ & 1.94$\times$ & 1.36$\times$ \\
ALFWorld 1.7B     & 2.05$\times$ & 2.07$\times$ & 1.28$\times$ \\
WebShop 3B        & 2.22$\times$ & 2.22$\times$ & 0.99$\times$ \\
WebShop 7B        & 1.77$\times$ & 1.92$\times$ & 0.83$\times$ \\
ScienceWorld 3B   & 1.62$\times$ & 1.63$\times$ & 1.22$\times$ \\
ScienceWorld 1.5B & 1.54$\times$ & 1.54$\times$ & 1.27$\times$ \\
\midrule
Geometric mean    & \textbf{1.84$\times$} & \textbf{1.87$\times$} & \textbf{1.14$\times$} \\
\bottomrule
\end{tabular*}
\end{table}

All six settings improve in trained-token and trainer-GPU efficiency.

\begin{figure}[ht]
\centering
\includegraphics[width=0.96\linewidth]{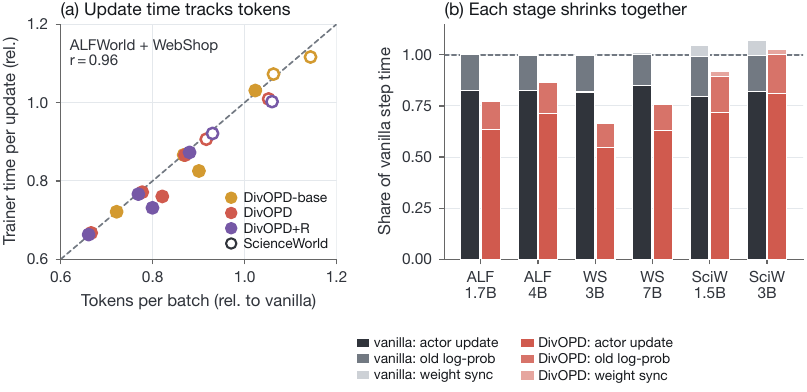}
\caption{\textbf{Cheaper updates come from fewer tokens.} (a) Relative trainer
time versus relative tokens per batch; hollow ScienceWorld points use the
trainer clock only here. (b) The stage-wise time breakdown.}
\label{fig:step-cost}
\end{figure}

\textbf{Fewer tokens make updates cheaper.} Trainer time closely tracks batch tokens on ALFWorld and
WebShop (Pearson $r=0.96$), while model FLOP utilization is nearly unchanged.
The gate removes long rows with truncated prompts, and the first-pass cap limits repeated turns
from long rollouts, reducing time across trainer stages
(Figure~\ref{fig:step-cost}).

\textbf{ScienceWorld gains come from fewer updates.} Because its turns are
valid and focus favors longer responses, its trainer-GPU improvement arises
primarily from reaching the target in fewer updates rather than from shorter
individual steps.

\textbf{Teacher-side cost.} The explorer scores every rollout turn, but many
scored tokens never enter a learner batch (Table~\ref{tab:teacher-cost}).
The explorer produces turns faster than the learner consumes them, and the
staleness rule limits how long excess turns can wait in the buffer.

\begin{table}[!htbp]
\caption{\textbf{Teacher tokens scored per token trained.} Effective tokens;
``not selected'' is the share never entering an optimizer batch. This audit
uses the $k=3$ runs from Table~\ref{tab:k-sweep}, not the per-setting caps of
Table~\ref{tab:main-results}. $^*$Two parallel 150-update trainer histories
sharing one buffer, merged.}
\label{tab:teacher-cost}
\centering\small
\setlength{\tabcolsep}{5pt}
\lighttablestyle
\begin{tabular*}{\linewidth}{@{\extracolsep{\fill}}lrrrr@{}}
\toprule
& \multicolumn{2}{c}{\textbf{DivOPD} ($k{=}3$)} & \multicolumn{2}{c}{\textbf{DivOPD-focus}} \\
\cmidrule(lr){2-3}\cmidrule(lr){4-5}
\textbf{Setting} & scored/trained & not selected & scored/trained & not selected \\
\midrule
ALFWorld 1.7B     & 2.77$\times$ & 63.9\% & 2.89$\times$ & 65.4\% \\
ALFWorld 4B       & 2.73$\times$ & 63.4\% & 2.93$\times$ & 65.9\% \\
ScienceWorld 1.5B & 3.66$\times$ & 72.6\% & 3.79$\times$ & 73.6\% \\
ScienceWorld 3B   & 3.41$\times$ & 70.6\% & 3.60$\times^*$ & 72.2\% \\
WebShop 3B        & 1.89$\times$ & 47.1\% & 1.86$\times$ & 46.3\% \\
WebShop 7B        & 1.95$\times$ & 48.7\% & 1.97$\times$ & 49.3\% \\
\midrule
Token-weighted    & 2.44$\times$ & 59.1\% & 2.60$\times$ & 61.6\% \\
\bottomrule
\end{tabular*}
\end{table}

\textbf{Selection overhead.} The composer sorts at most $cB=256$ rows and
adds no GPU work. As learner updates become faster, rollout generation can
limit further wall-time gains.

\FloatBarrier
\section{Additional Ablations and Teacher Recovery}
\label{app:ablation}

This appendix compares the method variants, tests the initial cap, and
details how teacher recovery extends DivOPD to no-progress rollouts.

\begin{table}[!htbp]
\caption{\textbf{Matched candidate access.} Both arms draw from the same
$256$-turn stream; vanilla otherwise reads in arrival order. Best-5 averages
the five best checkpoints; Final is SR at the common learner-version cutoff.
These separate runs are not numerically comparable to Table~\ref{tab:main-results}.}
\label{tab:access-control}
\centering
\fontsize{8.3}{9.2}\selectfont
\setlength{\tabcolsep}{4.0pt}
\lighttablestyle
\begin{tabular*}{\linewidth}{@{\extracolsep{\fill}}llrrrrrr@{}}
\toprule
& & \multicolumn{3}{c}{\textbf{Vanilla reader}}
& \multicolumn{3}{c}{\textbf{Rollout-first selection}} \\
\cmidrule(lr){3-5}\cmidrule(lr){6-8}
\textbf{Environment} & \textbf{Model}
& \textbf{Peak}$\uparrow$ & \textbf{Best-5}$\uparrow$ & \textbf{Final}$\uparrow$
& \textbf{Peak}$\uparrow$ & \textbf{Best-5}$\uparrow$ & \textbf{Final}$\uparrow$ \\
\midrule
ALFWorld     & 1.7B & 78.57 & 73.86 & \textbf{78.57} & \textbf{80.00} & \textbf{77.71} & \textbf{78.57} \\
WebShop      & 3B   & 75.00 & 68.44 & \textbf{71.09} & \textbf{78.12} & \textbf{74.06} & 68.75 \\
ScienceWorld & 1.5B & 73.44 & 72.66 & 73.05 & \textbf{77.73} & \textbf{74.38} & \textbf{74.61} \\
\bottomrule
\end{tabular*}
\end{table}

\begin{table}[!htbp]
\caption{\textbf{Matched-access cap comparison.} Both arms use the same
candidate-pool size, gate, ordering, pending rule, and within-rollout sampling;
only the initial cap differs. SR is in percent.}
\label{tab:matched-access-cap}
\centering
\fontsize{8.3}{9.2}\selectfont
\setlength{\tabcolsep}{4.0pt}
\lighttablestyle
\begin{tabular*}{\linewidth}{@{\extracolsep{\fill}}llrrrr@{}}
\toprule
& & \multicolumn{2}{c}{\textbf{Finite $k$ (DivOPD-base)}}
& \multicolumn{2}{c}{\textbf{$k=\infty$ (no cap)}} \\
\cmidrule(lr){3-4}\cmidrule(lr){5-6}
\textbf{Environment} & \textbf{Model}
& \textbf{Peak SR}$\uparrow$ & \textbf{Final-5 SR}$\uparrow$
& \textbf{Peak SR}$\uparrow$ & \textbf{Final-5 SR}$\uparrow$ \\
\midrule
ALFWorld     & 1.7B & 74.29 & 72.86 & 73.57 & 65.14 \\
ALFWorld     & 4B   & 92.14 & 88.86 & 90.00 & 87.71 \\
WebShop      & 3B   & 78.12 & 73.12 & 75.00 & 65.47 \\
WebShop      & 7B   & 88.28 & 78.44 & 87.62 & 81.09 \\
ScienceWorld & 1.5B & 70.31 & 67.50 & 67.97 & 62.89 \\
ScienceWorld & 3B   & 86.72 & 82.66 & 83.33 & 80.67 \\
\bottomrule
\end{tabular*}
\end{table}

\textbf{Recovery protocol.} Following the use of teacher intervention in
online imitation learning~\citep{ross2011reduction}, DivOPD$+$R invokes the
teacher after $p$ consecutive no-progress turns: an invalid action, or a
repeated action with an unchanged set of allowed actions. The teacher acts for at most
four to six turns per rollout. The patience/cap/warmup triplets (warmup in learner
updates) are ALFWorld 1.7B: $4/6/0$, 4B: $6/6/0$; WebShop 3B: $4/5/45$,
7B: $4/4/45$; and ScienceWorld 1.5B: $2/6/120$, 3B: $2/6/50$.
Teacher actions are excluded from the distillation loss, and the student resumes
interaction from the resulting state. Recovery is disabled at evaluation.

\begin{figure}[ht]
\centering
\includegraphics[width=0.96\textwidth]{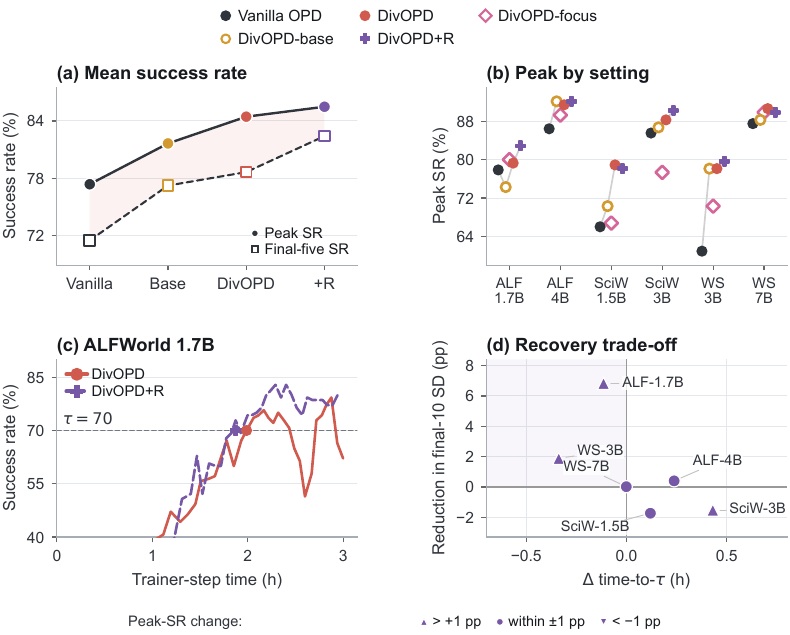}
\caption{\textbf{Variant chain and optional recovery.}
(a--b) Base combines gating, expanded access, and coverage; hollow diamonds
show the separate focus-only reference. (c) Recovery on ALFWorld 1.7B.
(d) Per-setting recovery effects, with marker shape encoding peak-SR change.
Only the controls in Figure~\ref{fig:component-controls} isolate individual
selection components.}
\label{fig:ablation-recovery}
\end{figure}

Table~\ref{tab:matched-access-cap} isolates the first-pass
cap under matched candidate access. Table~\ref{tab:k-sweep} complements that
comparison by varying the finite cap in DivOPD, with focus enabled.

\begin{table}[!htbp]
\caption{\textbf{Sensitivity to the initial per-rollout cap $k$.} Peak SR (\%);
TCOD-F2B is a reference. Best DivOPD values are highlighted, and second-best
values are underlined.}
\label{tab:k-sweep}
\centering
\fontsize{8.1}{8.8}\selectfont
\setlength{\tabcolsep}{4.2pt}
\renewcommand{\arraystretch}{0.98}
\begin{tabular*}{\linewidth}{@{\extracolsep{\fill}}lrrrrrr@{}}
\toprule
\textbf{Method}
& \multicolumn{2}{c}{\textbf{ALFWorld}}
& \multicolumn{2}{c}{\textbf{WebShop}}
& \multicolumn{2}{c}{\textbf{ScienceWorld}} \\
\cmidrule(lr){2-3}\cmidrule(lr){4-5}\cmidrule(lr){6-7}
& \textbf{1.7B} & \textbf{4B}
& \textbf{3B} & \textbf{7B}
& \textbf{1.5B} & \textbf{3B} \\
\midrule
TCOD-F2B
& 72.86 & 87.14 & 67.19 & 89.06 & 71.88 & 87.89 \\
\specialrule{0.35pt}{2.5pt}{0.8pt}
DivOPD ($k=3$)
& \tbest{87.14} & \snd{90.00} & 71.09 & \tbest{91.41} & \snd{77.73} & 85.55 \\
DivOPD ($k=4$)
& 79.29 & \tbest{91.43} & \snd{73.44} & \snd{90.62} & \tbest{78.91} & \snd{86.33} \\
DivOPD ($k=5$)
& 79.29 & \tbest{91.43} & \tbest{78.12} & 88.28 & 75.78 & \tbest{88.28} \\
DivOPD ($k=6$)
& \snd{80.71} & 88.57 & \snd{73.44} & 86.72 & 75.78 & 85.55 \\
\bottomrule
\end{tabular*}
\end{table}

\end{document}